\documentclass{article}
\usepackage{arxiv}

\usepackage[utf8]{inputenc} 
\usepackage[T1]{fontenc}    
\usepackage{hyperref}       
\usepackage{url}            
\usepackage{booktabs}       
\usepackage{amsfonts}       
\usepackage{nicefrac}       
\usepackage{microtype}      
\usepackage{lipsum}
\usepackage{graphicx}
\graphicspath{ {./images/} }

\usepackage{times}
\usepackage{soul}
\usepackage{url}

\usepackage[small]{caption}
\usepackage{graphicx}
\usepackage{amsmath}
\usepackage{amsthm}
\usepackage{booktabs}
\usepackage{algorithm}
\usepackage{algorithmic}
\usepackage[switch]{lineno}

\usepackage{bm}
\usepackage{amssymb}

\usepackage{threeparttable, array, float} 
\usepackage{color, colortbl}
\usepackage[capitalize]{cleveref}
\usepackage{nicefrac} 
\usepackage{xcolor}
\usepackage{graphicx}
\usepackage{float}
\usepackage{newfloat}
\usepackage{subcaption}
\usepackage{mathtools}
\usepackage{indentfirst}
\usepackage{natbib}

\definecolor{dark-red}{rgb}{0.4, 0.15, 0.15}
\definecolor{dark-blue}{rgb}{0.15, 0.15, 0.4}
\definecolor{medium-red}{rgb}{0.5, 0, 0}
\definecolor{medium-blue}{rgb}{0, 0, 0.5}
\definecolor{light-red}{rgb}{0.7, 0, 0}
\definecolor{light-blue}{rgb}{0, 0, 0.7}

\definecolor{red}{HTML}{E51400} 
\definecolor{blue}{HTML}{0050EF} 
\definecolor{green}{HTML}{008A00} 
\definecolor{purple}{HTML}{AA00FF} 
\definecolor{orange}{HTML}{FF7F00}
\definecolor{gray}{HTML}{848482}
\definecolor{Gray}{gray}{0.85}
\definecolor{LightGray}{gray}{0.96}

\newtheorem{theorem}{Theorem}[section]

\newtheorem{assumption}{Assumption}[section]

\newtheorem{lemma}{Lemma}[section]

\title{Efficient Linear Bandits via Cluster-Aware Sketching}
\author{
Hantao Yang \\
  University of Science and Technology of China\\
  \texttt{yanghantao@mail.ustc.edu.cn} \\
   \And
Hong Xie\\
  University of Science and Technology of China\\
  \texttt{xiehong2018@foxmail.com} \\
  \And
Defu Lian\\
  University of Science and Technology of China\\
  \texttt{liandefu@ustc.edu.cn} \\
}

\begin{document}
\maketitle
\begin{abstract}
We study the problem of computational efficiency for linear bandits in high-dimensional settings with a finite arm set. In linear bandits, the increase in the dimension $d$ of the feature vectors leads to growing computational costs of $O(d^2)$ at each round of update. Traditional sketching-based methods such as SOFUL reduce computation via fixed-size matrix sketching, yet run the risk of incurring vacuous linear regret when the spectral tail of the data is heavy and the sketch size is inadequately selected. To guarantee regret convergence and effectively reduce computational costs, we introduce a clustering mechanism and propose the Cluster Sketch Linear Bandit (CS-LB) algorithm. Our method preserves the full covariance information in each cluster to guarantee robust sublinear regret without spectral-tail vulnerabilities, performs cluster switching by assigning a sentinel for each cluster, and reduces per-round update computation to $O(l^2d)$ via a tunable sketch size $l<d$. Experiments on synthetic datasets demonstrate that our method consistently maintains a favorable trade-off between efficiency and regret.

\end{abstract}


\section{Introduction}
The stochastic linear contextual bandit is a fundamental framework for sequential decision-making under uncertainty, where an agent repeatedly selects an action from a finite set based on its associated context vector, and observes a noisy reward that is linear in the unknown parameter. This model has found widespread applications in personalized recommendation \cite{LinUCB}, online advertising \cite{10.1145/2505515.2514700}, and mobile health \cite{Tewari2017FromAT}, among others. Classical algorithms such as LinUCB \cite{ContextualBanditsLinearPayoff,OFUL} and linear Thompson Sampling \cite{Thompsonsamplingbandits} achieve near-optimal regret guarantees, but incur a prohibitive $O(d^2)$ per-round time complexity due to the need to maintain and invert a $O(d\times d)$ covariance matrix, where $d$ is the feature dimension.

As data scales have grown rapidly in recent years, both the feature dimension $d$ and the number of arms $N$ can be extremely large, rendering exact methods computationally infeasible. To reduce this update computation, recent works have turned to matrix sketching techniques, which maintain a low-dimensional approximation of the covariance matrix. SOFUL \cite{SOFUL} applies Frequent Directions \cite{FD} to compress the covariance matrix from dimension $d$ to $l < d$, reducing per-round complexity to $O(ld)$. The CBSCFD \cite{CBSCFD} algorithm improves upon SOFUL by leveraging the accumulated information of its truncated singular values, thereby achieving better regret performance. While these methods offer substantial computational speedups, they suffer from a critical drawback: the global sketching of the entire action-history covariance matrix inevitably introduces approximation errors that accumulate over time. This spectral distortion can inflate the confidence ellipsoid, degrade arm-selection accuracy, and in some cases lead to linear regret or even non-convergence, particularly when the covariance spectrum has a heavy tail. In other words, existing sketching methods sacrifice linear regression accuracy in favor of computational efficiency. However, this trade-off struggles to guarantee the convergence of regret, which may render them ineffective in systems that require stability guarantees. \cite{DBSLinUCB} improved upon SOFUL and proposed DBSLinUCB, which ensures sublinear regret through a multi-scale, adaptively adjusted sketch size. Although DBSLinUCB resolves the non-convergence issue of SOFUL, its theoretical guarantees rely on adaptive sketch-size growth and cannot be extended to the fixed-sketch-size setting. Moreover, its computational efficiency is predicated on low-rank data; on full-rank data, it reverts to $O(d^2)$ complexity, losing the speed advantage of sketching.

To address the demand for regret convergence in computation-constrained linear bandit settings, we propose the Cluster Sketch Linear Bandit (CS‑LB) algorithm. To the best of our knowledge, CS‑LB is the first sketching‑based bandit strategy that simultaneously reduces computational cost and guarantees regret convergence under a fixed sketch size. Our main contributions are summarized as follows:
\begin{itemize}
    \item \textbf{Cluster Sketching Framework.} We propose a clustering‑assisted sketching framework for finite‑action linear bandits. Unlike prior global sketching approaches, our method partitions the action set into multiple clusters and maintains a dedicated sketch for each cluster. This design preserves the information integrity of each cluster’s sketch matrix, which serves as the foundation for our regret convergence guarantee.
   \item \textbf{Sentinel‑Guided Cluster Selection.} We introduce a sentinel mechanism atop the cluster‑sketch structure, where the sentinel of each cluster is defined as the maximum upper confidence bound among its arms. This mechanism enforces global optimism in the spirit of OFUL, aligning the algorithm’s behavior with OFUL and enabling a clean theoretical analysis of regret.
    \item \textbf{Regret Convergence and Computational Efficiency.} We prove that CS‑LB achieves a regret bound of $\widetilde{O}(l\sqrt{N_{\mathcal{A}}^{l}T})$ and a per‑round computational cost of $O(l^2d)$, excluding the additional overhead incurred during the warm-up phase, offering a clear efficiency gain when $l^2 < d$. Notably, unlike previous sketching‑based methods, our regret bound is free of the spectral error term $\Delta_T$—the primary source of linear regret in existing approaches—thus ensuring rigorous regret convergence. Moreover, our theoretical framework remains valid under both a fixed sketch size and full-rank data settings.
    \item \textbf{Empirical Validation.} Experiments on synthetic data show that CS‑LB consistently outperforms SOFUL and CBSCFD in cumulative regret, with the largest improvements observed when the action set exhibits a natural clustering structure. In addition, CS‑LB achieves lower running times than OFUL, confirming its practical computational advantage.
\end{itemize}

\section{Related Work}
\textbf{Linear Bandits}
Contextual bandits generalize the finite-arm setting by allowing the exploitation of side information through linear models, thereby giving rise to the linear bandit framework. The study of regret minimization in linear bandits was first studied by \cite{645528.657779}, and has attracted extensive interest in the development of various algorithms \cite{LinearlyParameterized,ContextualBanditsLinearPayoff}. \cite{944919.944941} provided an early algorithm (LinRel) and theoretical analysis for this problem by introducing confidence-bound-based techniques, improving the regret bound from the previous $\widetilde{O}(T^{3/4})$ to $\widetilde{O}(\sqrt{T})$. Subsequently, \cite{Dani2008StochasticLO} extended the theoretical framework to more general compact action sets. \cite{LinUCB} proposed the LinUCB algorithm and successfully applied it to personalized recommendation. \cite{ContextualBanditsLinearPayoff} conducted further theoretical investigations of linear bandits by maintaining a regularized least‑squares estimate of the unknown parameter and constructing a confidence ellipsoid centered on this estimate. By selecting the action that maximizes the upper confidence bound, their work demonstrated that LinUCB achieves a near‑optimal regret of $\widetilde{O}(d\sqrt{T})$ under relatively mild assumptions. \cite{OFUL} further refined the analysis by developing a tighter confidence set based on self-normalized martingale inequalities, and introduced the classic OFUL strategy. \cite{Thompsonsamplingbandits} adapted Thompson Sampling \cite{Thompson1933ONTL} to linear bandits, and this method was later further refined by \cite{Abeille2016LinearTS}. Linear Thompson sampling randomly draws a parameter from the posterior distribution and selects the action that maximizes the expected reward under this sample. Its performance is comparable to that of OFUL up to logarithmic factors.

\textbf{Sketching Bandits}
Sketching is a classic dimensionality reduction technique \cite{10.1145/3055399.3055431,pmlr-v80-andoni18a,BeyondJohnson} that has been introduced into linear bandits to reduce the computational cost of matrix inversion \cite{Yu_Lyu_King_2017,SOFUL}. \cite{FD} proposed Frequent Directions (FD), a sketching method with strong theoretical guarantees in streaming settings, making it particularly suitable for online environments. \cite{SOFUL} first combined FD with the classic OFUL algorithm in linear contextual bandits, yielding the SOFUL algorithm and reducing per-round update costs. Building on SOFUL, \cite{CBSCFD} developed a variant of FD for linear bandits that achieves a tighter regret bound. To overcome the potential non-convergence issue of traditional sketching methods, \cite{DBSLinUCB} proposed DBSLinUCB, a multi-scale sketching framework that adaptively adjusts the sketch size during learning, thus avoiding the linear regret trap of fixed-size sketches. Unlike previous single-scale approaches, DBSLinUCB dynamically distributes the sketch budget across blocks to guarantee sublinear regret without prior knowledge of the data spectrum. While DBSLinUCB features an adaptive sketch size, our work focuses on the fixed-sketch-size setting.

\section{Problem Setting}
\subsection{Notation and Definition}
In our paper, every d-dimensional vector is represented as a column vector. For matrix $\bm{A}$, $[\bm{A}]_k$ denotes the matrix formed by taking the first $k$ rows of matrix $\bm{A}$. For matrix $\bm{A}\in \mathbb{R}^{d_1\times d_2}$, the computed SVD of $\bm{A}$ is defined as $\bm{A} = \bm{U}\bm{\Sigma}\bm{V}^{\top}$, where $\bm{U}$ is $d_1\times d_2$ matrix, $\bm{\Sigma}$ is $d_2\times d_2$ diagonal matrix, $\bm{V}$ is $d_2\times d_2$ matrix. We use $SVD(\bm{A})$ to denote the singular value decomposition of matrix $\bm{A}$ and obtain the corresponding $\bm{U}$, $\bm{V}^{\top}$ and $\bm{\Sigma}$. Let $\sigma_i(\bm{A})$ be the $i$-th largest singular value of $\bm{A}$. For a positive integer $N$, let $[N]=\{1,...,N\}$. For a set $\mathcal{A}$, $|\mathcal{A}|$ denotes the number of elements contained in the set. For theoretical bounds, the notation $\widetilde{O}(\cdot)$ omits logarithmic factors.

\subsection{Stochastic Linear Bandits}
In the MAB model, the policy needs to select an appropriate arm to pull at each round; for the Stochastic Linear Bandits (SLB), the reward of each arm follows a linear function with unknown parameter $\theta^*\in \mathbb{R}^d$.

For classical linear bandits, the learner is given an arm set $\mathcal{A}=[N]$, which corresponds to known arm vectors $\{\bm{x}^{1},...,\bm{x}^{N}\}\subset\mathbb{R}^{d\times 1}$. In round $t$, the learner chooses an action $\bm{x}_t\in\mathcal{A}$ and receives reward $r_t=\left \langle \bm{x}_t,\bm{\theta}^* \right \rangle+\eta_t$, where $\eta_t$ is conditionally independent of $\bm{x}_t$ given $\bm{x}_{1:t-1},r_{1:t-1}$. For our model, we currently consider LSB in the finite-set setting.

According to \cite{OFUL}, we make the following assumption:
\begin{assumption}
The noise $\eta_t$ is R-sub-Gaussian conditioning on $\bm{x}_{1:t-1},r_{1:t-1}$.
\[
\mathbb{E}[e^{\lambda \eta_t}|\bm{x}_{1:t-1},r_{1:t-1}]\leq exp(R^2\lambda^2/2).
\]
\end{assumption}
\begin{assumption}
$\left \|\bm{\theta}^*  \right \|_2\leq S$, $\left \|\bm{x}  \right \|_2\leq L$ for all action $\bm{x}_t\in \mathcal{A}$, for all $t\in [T]$.
\end{assumption}
The goal of the linear bandits is to collaboratively minimize the cumulative regret defined as
\begin{align}
Regret(T)\nonumber
&:=\sum_{t=1}^{T}\left ( \max_{\bm{x}\in \mathcal{A}}\left \langle\bm{x},\bm{\theta}^*\right \rangle-\left \langle\bm{x}_t,\bm{\theta}^*\right \rangle \right )
\end{align}
where $\bm{x}^*=argmax_{\bm{x}\in \mathcal{A}}\left \langle\bm{x},\bm{\theta}^*\right \rangle$ is the optimal arm.

We use $\bm{X}_t=[\bm{x}_1,...,\bm{x}_t]^{\top}$ to define the matrix of all actions selected up to round $t$.

\subsection{Computation efficiency} 
Although the update overhead of linear bandits can be reduced to $O(d^2)$, such computational cost remains substantial in high-dimensional scenarios, especially in systems emphasizing real-time performance where the per-round computation budget may be strictly constrained. Unlike offline training pipelines that can leverage distributed computing and batching, online bandit algorithms demand immediate responses, and any computation beyond the allocated budget directly translates into service degradation or revenue loss. This practical constraint motivates the development of algorithms that offer a flexible trade-off between computational cost and statistical accuracy, allowing practitioners to adapt to varying system loads and latency requirements.
\subsection{Sketch}
In the study of computational overhead in linear bandits, \cite{SOFUL} first proposed Sketched OFUL (SOFUL), which effectively reduces the computational cost of per-round updates. In traditional linear bandits, the update of the linear regression strategy relies on computing the inverse of the matrix $\bm{V}_t=\lambda\bm{I}+\bm{X}_t^{\top}\bm{X}_t$. This updating always incurs $O(d^2)$ computation overhead. Let $l<d$ be the sketch size. The goal of sketching is to keep an approximate matrix $\bm{S}_T\in \mathbb{R}^{l\times d}$ in place of $\bm{X}_T$ for policy updates, where $\bm{S}_T$ satisfies
\[
\bm{S}_T^{\top}\bm{S}_T \approx \bm{X}_T^{\top}\bm{X}_T.
\]

The classical sketch algorithm Frequent Directions (FD) \cite{FD} truncates the smallest singular value after performing an SVD of the covariance matrix at every update. In each round $t$, FD updates as follows:
\begin{align}
\nonumber &[\bm{U},\bm{\Sigma},\bm{V}^{\top}]{=}SVD\left([\bm{S}_{t-1}^{\top};\bm{x}_t]^{\top}\right),\\
\nonumber &\delta_t{=}\sigma_l^2(\bm{\Sigma}), \bm{S}_{t}{=}\sqrt{\bm{\Sigma}^2-\delta_t\bm{I}}\bm{V}^{\top},
\end{align}
where $\sigma_i(\bm{\Sigma})$ is the $i$-th largest singular value of $\bm{\Sigma}$. 

FD method can reduce the computational cost from $O(d^2)$ to $O(ld)$ at the expense of doubling the space used by the algorithm. The FD method can effectively reduce the computational cost; however, it incurs information loss due to singular value truncation, thereby leading to non-convergence. The sketch matrix serves as an approximate representation of the original covariance matrix; some information is lost during the sketching process (e.g., the truncation of singular values), resulting in a discrepancy between the sketch and the original covariance matrix.

To address the information loss and instability caused by FD, a variant called Spectral Compensation Frequent Directions (SCFD) \cite{CBSCFD} introduces a compensation mechanism that preserves the positive definite monotonicity of the covariance sequence. Unlike FD, which directly discards the truncated spectral mass $\delta_t$, SCFD accumulates this lost information into a scalar $\alpha_t = \alpha_{t-1} + \delta_t$ and explicitly adds it as a diagonal correction $\alpha_t \bm{I}_d$ to the sketched approximation. That is, SCFD maintains
\[
\hat{\bm{V}}_t = \bm{S}_t^{\top}\bm{S}_t + \alpha_t \bm{I}_d\approx \bm{X}_T^{\top}\bm{X}_T+\lambda\bm{I},
\]
which ensures $\hat{\bm{V}}_t \succeq \hat{\bm{V}}_{t-1}$ for all $t$, and achieves better-conditioned approximations. Although SCFD achieves improved regret bounds and stronger robustness, it still retains the same error bound as FD. The non‑convergence issue caused by heavy spectral tails in sketching methods persists.

\section{Algorithm}
In this section, we analyze the limitations of existing approaches for improving computational efficiency in linear bandits, and present our proposed algorithm, Cluster Sketch Linear Bandits (CS-LB). The CS-LB algorithm not only enhances computational efficiency with a fixed sketch size, but also ensures the convergence of the corresponding regret.
\subsection{Limitations of Existing Approaches}
\textbf{OFUL \cite{OFUL}.} As a classic algorithm for linear bandits, OFUL enjoys rigorous theoretical guarantees on regret convergence. However, its per-round update cost is prohibitively large in high-dimensional settings, which makes it difficult to deploy OFUL in computation-constrained environments.

\textbf{SOFUL \cite{SOFUL} / CBSCFD \cite{CBSCFD}.} SOFUL reduces OFUL's computation to $O(dl)$ via fixed-size matrix sketching. Its regret bound, however, depends critically on the spectral error $\Delta_T$:
\[
\mathrm{Regret}_T^{SOFUL} = \tilde{O}\left((1 + \Delta_T)^{3/2} \left(l + d\log(1 + \Delta_T)\right) \sqrt{T}\right).
\]
When the data exhibits heavy spectral tails, $\Delta_T$ grows with $T$, and the regret collapses to linear. Thus SOFUL trades statistical reliability for efficiency—a compromise that fails precisely when robust performance is most needed. The same issue also appears in CBSCFD. Even though CBSCFD reduces the spectral error term $O(\Delta_T^{3/2})$ in SOFUL to $O(\sqrt{\Delta_T})$, it still carries the risk of linear regret \cite{DBSLinUCB}.

\textbf{DSBLinUCB \cite{DBSLinUCB}.} DBSLinUCB employs a multi-scale, adaptively adjusted sketch size to ensure that the global spectral error is bounded by a fixed parameter $\varepsilon$, thereby guaranteeing sublinear regret in all cases. This effectively addresses the long-standing linear regret issue inherent in SOFUL and CBSCFD. However, DBSLinUCB still has several limitations. First, it maintains gradually growing blocks to adaptively increase the sketch size so as to ensure regret convergence, which makes it inapplicable in computation-constrained scenarios where a fixed sketch size is required. Second, DBSLinUCB is built on the assumption that the data is low-rank. When the data is full-rank and heavy-tailed, DBSLinUCB degrades to exact rank-1 updates, in which case its time complexity increases to $O(d^2)$—the same as OFUL—thereby losing the speed advantage of sketching.

\subsection{Cluster Sketch Linear Bandits}
We address the regret non-convergence problem in traditional sketching methods and propose the Cluster Sketch Linear Bandits (CS-LB) algorithm to improve the computational efficiency of updates for linear bandits with a finite arm set, as shown in Algorithm \ref{alg:CS-LB}. Moreover, our method still guarantees regret convergence even in the presence of full-rank data and under a fixed sketch size. CS-LB is primarily composed of the following components:

\textbf{FD Sketching.} As mentioned above, FD is a deterministic sketching algorithm that maintains the $l$-th smallest singular value of the sketch matrix and truncates the remaining singular values. This approach strictly bounds the computational complexity of SVD decomposition at $O(l^2d)$, and by leveraging the acceleration techniques in FD, the average per-round update cost can be further reduced to $O(ld)$, where the sketch size $l < d$. While this method achieves a significant reduction in computational overhead, when the matrix $X_T$ is not low-rank, the FD method discards much of the information necessary for accurately computing linear regression, ultimately leading to non-convergent regret in SOFUL \cite{SOFUL} and CBSCFD \cite{CBSCFD}. Although the FD method may suffer from non-convergence, its simplicity and efficiency lead us to adopt it for reducing the update cost of linear bandits, with further optimizations built on top.  Here we make some modifications to the update procedure in \cite{FD}, and present the resulting FD sketching update procedure in Algorithm \ref{alg:FD}.

\begin{algorithm}[!ht] 
    \caption{FD sketching: $FD(\bm{S}, \bm{x})$} 
	\label{alg:FD}
    \textbf{Input:} $\bm{S}$, $\bm{x}$\\ 
    \textbf{Parameters:} sketch size $l$\\ 
    \textbf{Output:} 
    \begin{algorithmic}[1]
    \STATE compute SVD: $[\bm{U},\bm{\Sigma},\bm{V}^T]=SVD(\text{vstack}[\bm{S}, \bm{x}^{\top}])$\\
	where $\bm{U}\in \mathbb{R}^{(l+1)\times (l+1)},\bm{\Sigma}\in \mathbb{R}^{(l+1)\times d},\bm{V}^T\in \mathbb{R}^{d\times d}$
    \STATE $k=rank(\bm{\Sigma})$
	\STATE $\hat{\bm{\Sigma}}=[\bm{\Sigma}]_{l\times d}$
	\STATE $\bm{S}'=\sqrt{\hat{\bm{\Sigma}}}\bm{V}^T\in \mathbb{R}^{l\times d}$
	\STATE $\bm{H}=([\hat{\bm{\Sigma}}]_{l\times l}+\lambda\bm{I})^{-1}\in \mathbb{R}^{l\times l}$
    \RETURN $\bm{S}', \bm{H}, k$
    \end{algorithmic}
    \end{algorithm}

The vstack function performs row-wise concatenation of the sketch matrix $\bm{S}$ and the vector $\bm{x}$. The FD sketch algorithm considered herein receives an $l\times d$ sketch matrix and a $d\times 1$ update vector $\bm{x}$ as inputs, and outputs the rank $k$ resulting from the concatenation, the updated $l\times d$ sketch matrix $S'$, as well as the matrix $H'$ needed for inverse computation. Let $\bar{\bm{V}}=\lambda\bm{I}+\bm{S'}^{\top}\bm{S'}$. According to \cite{SOFUL}, $\bar{\bm{V}}_t^{-1}=\frac{1}{\lambda}(\bm{I}-\bm{S'}^{\top}\bm{H}\bm{S'})$ can be efficiently computed.

\textbf{Cluster sketching Warm-up.}
We propose a clustering strategy that integrates clustering with matrix sketching, where each cluster maintains an independent sketch matrix to preserve the local covariance structure of arms within the cluster, rather than relying on a global unified sketch. We partition the original arm set into clusters based on the rank obtained from FD sketch updates, ensuring that the rank of the matrix composed of arm vectors within each cluster does not exceed $l$. During this clustering process, an arm is assigned to a cluster if its inclusion does not cause the cluster's rank to exceed $l$; otherwise, we search for another suitable cluster or let the arm form a new cluster on its own.

Algorithm \ref{alg:Warm-up} presents our cluster sketch partitioning scheme. We use $\mathcal{C}^{all}$ to denote the set of all clusters. Each cluster $\mathcal{C}$ is a disjoint subset of the arm set, and each cluster maintains its own sketch matrix $\mathbf{S}_{\mathcal{C}}$. Algorithm \ref{alg:Warm-up} takes an arm set $\mathcal{A}$ as input and returns a partitioned cluster set $\mathcal{C}^{all}$ after $t=N$ rounds.

In the warm-up phase, FD sketching update process itself corresponds to computing the rank within the cluster (i.e. the matrix $\bm{X}_{\mathcal{C},t}$ obtained by concatenating the arms contained in the cluster as row vectors). By leveraging FD sketching, the rank of $\bm{X}_{\mathcal{C},t}$ can be obtained with a computational cost of $O(l^2d)$. Through cluster sketch-based warm-up of the arm set, the rank of the matrix $\bm{X}_{\mathcal{C},N}$ is bounded by the sketch size $l$. Drawing upon Property 3 in \cite{SOFUL}, we ensure that $\bm{S}_{\mathcal{C},N}^{\top}\bm{S}_{\mathcal{C},N} = \bm{X}_{\mathcal{C},N}^{\top}\bm{X}_{\mathcal{C},N}$ holds within each cluster. Consequently, the FD sketching update incurs no information loss inside any cluster, thereby guaranteeing the convergence of regret.

 \begin{algorithm}[!ht] 
    \caption{Cluster Sketching Warm-up: $CSW(\mathcal{A},l)$} 
	\label{alg:Warm-up}
    \textbf{Input:} $\mathcal{A}$, sketch size $l$\\ 
    \textbf{Initialization:} $\mathcal{C}^{all}=\phi$, Find=False, $t=0$\\
    \textbf{Output:} 
    \begin{algorithmic}[1]
    \FOR{$\bm{x}\in \mathcal{A}$}
    \FOR{$\mathcal{C}\in \mathcal{C}^{all}$}
    \STATE $\bm{S}_{\mathcal{C},t+1},\bm{H}_{\mathcal{C},t+1},k_{\mathcal{C},t+1}=FD(\bm{S}_{\mathcal{C},t},\bm{x})$
    \IF{$k_{\mathcal{C},t+1}\leq l$}
    \STATE Find=True, $\mathcal{C}=\mathcal{C}\cup \{\bm{x}_t\}$
    \STATE break
    \ELSE
    \STATE Find=False, $\bm{S}_{\mathcal{C},t+1}=\bm{S}_{\mathcal{C},t}$
    \ENDIF
    \ENDFOR
    \IF{not Find}
    \STATE $\mathcal{C}_{new}=\{\bm{x}\}$
     \STATE $\bm{S}_{\mathcal{C}_{new},t+1}=\text{vstack}[\bm{x}^{\top},\bm{0}_{(l-1)\times d}]\in \mathbb{R}^{l\times d}$.
     \STATE $\mathcal{C}^{all}.\text{append}(\mathcal{C}_{new})$, $\text{len}(\mathcal{C}^{all}){+}{=}1$
    \ENDIF
    \STATE $t{+}{=}1$
    \ENDFOR
    \RETURN $\mathcal{C}^{all}$
    \end{algorithmic}
    \end{algorithm}

\textbf{Confidence Ellipsoid.} Under our cluster sketch partitioning strategy, FD sketching updates within each cluster incur no information loss. Consequently, the linear regression estimates $\hat{\bm{\theta}}_{\mathcal{C},t}$ within each cluster can accurately identify the optimal arm corresponding to that cluster. Following \cite{OFUL} and \cite{SOFUL}, we construct the confidence interval as:
\begin{align}
\label{eq:beta}
\beta_t(\delta)=R\sqrt{l\ln(1+\frac{tL^2}{l\lambda})+2\ln\frac{N}{l\delta}}+S\sqrt{\lambda}.
\end{align}
\begin{lemma}
\label{lem:main1}
With probability at least $1-\delta$, for any $0< t\leq T$ and $\mathcal{C}\in \mathcal{C}^{all}$, we have
\[
\left \|\hat{\bm{\theta}}_{\mathcal{C},t}-\bm{\theta}^*  \right \|_{\bm{\bar{V}}_{\mathcal{C},t}}\leq \beta_{t}(\delta),
\]
where $\bm{\bar{V}}_{\mathcal{C},t}=\lambda\bm{I}+\bm{S}_{\mathcal{C},t}^{\top}\bm{S}_{\mathcal{C},t}$ is the sketched regularized correlation matrix for cluster $\mathcal{C}$ in round $t$.
\end{lemma}

\textbf{Sentinel Strategy.} Combining with the clustering partition, we propose CS-LB (Algorithm \ref{alg:CS-LB}) to perform arm selection and update at each round. We assign a sentinel $U(\mathcal{C})$ to each cluster, initialized to positive infinity (or the upper bound on rewards). During subsequent updates, the sentinel $U(\mathcal{C})$ is set to the maximum upper confidence bound of the expected reward among all arms within that cluster. 

At the beginning of each round, the algorithm selects the cluster with the largest sentinel value as the current active cluster $\mathcal{W}_t$, and then performs LinUCB selection within $\mathcal{W}_t$. Specifically, it chooses the current arm $\bm{x}_t$ with the maximum upper confidence bound of the linear estimate to pull. Subsequently, FD sketching is updated to incorporate $\bm{x}_t$ into the sketch matrix $\bm{S}_{\mathcal{W}_t,t}$. Note that after warm-up, the rank within each cluster remains below $l$, and the selected arm belongs to the current cluster. Therefore, the FD sketching update incurs no information loss, and within the current cluster $\mathcal{W}_t$, it effectively behaves as a rank‑$l$ LinUCB update. Moreover, owing to the adoption of the FD sketching update, the inverse matrix $\bar{\bm{V}}_{\mathcal{W}_t,t}^{-1}$ can be obtained via matrix multiplication without the need for separate computation. According to \cite{SOFUL}, the computational cost per round stems from the SVD decomposition, yielding a per-round complexity of $O(l^2d)$. Compared with the $O(d^2)$ cost incurred by computing the inverse matrix in the traditional OFUL algorithm, CS-LB achieves a significant reduction in computational overhead when $l<\sqrt{d}$.

The sentinel strategy allows our method to preserve the OFUL policy at the global level, while enabling the adoption of the LinUCB policy within individual clusters after warm-up. The integration of these two strategies constitutes our CS-LB algorithm. Overall, the CS-LB algorithm maintains the optimism‑in‑the‑face‑of‑uncertainty principle globally and executes a rank‑$l$ LinUCB method inside each cluster, thereby substantially lowering the per‑round computational cost without compromising regret convergence.

\begin{algorithm}[!ht] 
    \caption{Cluster Sketch Linear Bandits (CS-LB)} 
	\label{alg:CS-LB}
    \textbf{Input:} $\mathcal{A}$, sketch size $l$, $\delta$\\ 
    \textbf{Initialization:}  $\mathcal{C}^{all}\longleftarrow CSW(\mathcal{A},l)$\\
    $U(\mathcal{C}){=}+\infty, \bm{S}_{\mathcal{C},1}{=}\bm{0}_{l\times d},\hat{\bm{\theta}}_{\mathcal{C},1}{=}\bm{0}_{1\times d},\bar{\bm{V}}_{\mathcal{C},1}^{-1}{=}\frac{1}{\lambda}\bm{I}_{d\times d}, \forall \mathcal{C}{\in}\mathcal{C}^{all}$\\
    \textbf{Output:} 
    \begin{algorithmic}[1]
    \FOR{round $t=1,...,T$}
    \STATE $\mathcal{W}_t=argmax_{\mathcal{C}\in \mathcal{C}^{all}}U(\mathcal{C})$
    \STATE $\bm{x}_t=argmax_{\bm{x}\in\mathcal{W}_t}\left\{\bm{x}^{\top}\bm{\hat{\theta}}_{\mathcal{W}_{t},t}+\beta_t(\delta)\left\|\bm{x} \right\|_{\bar{\bm{V}}_{\mathcal{W}_{t},t}^{-1}}\right\}$
    \STATE Play $\bm{x}_t$ and receive $y_t$
    \STATE $\bm{S}_{\mathcal{W}_t,t+1},\bm{H}_{\mathcal{W}_t,t+1},k_{\mathcal{W}_t,t+1}=FD(\bm{S}_{\mathcal{W}_t,t},\bm{x}_t)$
    \STATE $\bar{\bm{V}}_{\mathcal{W}_t,t+1}^{-1}=\frac{1}{\lambda}\left (\bm{I}-\bm{S}_{\mathcal{W}_t,t+1}^T\bm{H}_{\mathcal{W}_t,t+1}\bm{S}_{\mathcal{W}_t,t+1} \right )$
    \STATE $\bm{\beta}_{\mathcal{W}_t,t+1}=\bm{\beta}_{\mathcal{W}_t,t}+y_t\bm{x}_t$
    \STATE $\hat{\bm{\theta}}_{\mathcal{W}_t,t+1}=\bar{\bm{V}}_{\mathcal{W}_t,t+1}^{-1}\bm{\beta}_{\mathcal{W}_t,t+1}$
    \STATE $U(\mathcal{W}_t){=}\max_{\bm{x}\in\mathcal{W}_t}\left\{\bm{x}^{\top}\bm{\hat{\theta}}_{\mathcal{W}_{t},t+1}{+}\beta_{t+1}(\delta)\left\|\bm{x} \right\|_{\bar{\bm{V}}_{\mathcal{W}_{t},t+1}^{-1}}\right\}$
    \FOR{$\mathcal{C}\in \mathcal{C}^{all}$ and $\mathcal{C}\ne \mathcal{W}_t$}
     \STATE $\bm{S}_{\mathcal{C},t+1}=\bm{S}_{\mathcal{C},t},\bm{H}_{\mathcal{C},t+1}=\bm{H}_{\mathcal{C},t},k_{\mathcal{C},t+1}=k_{\mathcal{C},t}$.
    \STATE $\bar{\bm{V}}_{\mathcal{C},t+1}^{-1}=\bar{\bm{V}}_{\mathcal{C},t}^{-1}, \bm{\beta}_{\mathcal{C},t+1}=\bm{\beta}_{\mathcal{C},t}, \hat{\bm{\theta}}_{\mathcal{C},t+1}=\hat{\bm{\theta}}_{\mathcal{C},t}$
    \ENDFOR
    \ENDFOR 
    \end{algorithmic}
    \end{algorithm}

\section{Theoretical Results}
\subsection{Regret Analysis}
\begin{theorem} 
\label{the:main}
Given the sketch size $l$, if $\beta_t(\delta)$ is chosen as specified in Equation \ref{eq:beta} and set $\lambda\geq\max\{1,L^2\}$, then with probability $1-\delta$, the expected cumulative regret of Algoritm \ref{alg:CS-LB} is upper bounded by
\begin{align}
    Regret(T)
    &\leq 2\beta_T(\delta)\sqrt{2lN_{\mathcal{A}}^{l}T\ln\left(1+\frac{TL^2}{\lambda l}\right)} \\
    &=\widetilde{O}\left(l\sqrt{N_{\mathcal{A}}^{l}T}\right)
\end{align}
where $N_{\mathcal{A}}^{l}=|\mathcal{C}^{all}|$ is the number of clusters produced by Cluster sketch warm-up on arm set $\mathcal{A}$ with sketch size $l$.
\end{theorem}

Compared with the results of SOFUL/CBSCFD in the finite arm set setting, our regret bound eliminates the influence of the spectral error term $\Delta_T$, rigorously guaranteeing the convergence of regret for any sketch size $l>1$. This leads to an overall improvement in the algorithm's performance. The detailed proof is provided in the Appendix.

Since each cluster has a rank of $l$, when the arms are processed sequentially according to Algorithm \ref{alg:Warm-up}, each cluster can contain at least $l$ arms. Therefore, we have $N_{\mathcal{A}}^{l}<\left \lceil N/l \right \rceil $. In this case, the regret bound of CS-LB is obtained as $\widetilde{O}(\sqrt{lNT})$, which implies that in the worst case, the regret of CS-LB depends on the data scale $N$. Here, $N_{\mathcal{A}}^{l}$ is a quantity determined by the composition of the dataset $\mathcal{A}$ and the sketch size $l$. If the dataset exhibits strong linear correlation, the number of arms contained in each cluster increases, and consequently the number of cluster $N_{\mathcal{A}}^{l}$ becomes substantially smaller. For datasets with a strong linear structure, by choosing an appropriate $l$, it is possible to achieve $l\sqrt{N_{\mathcal{A}}^{l}}<d$, thereby reducing computational overhead while attaining better regret performance than OFUL. We will demonstrate this in the subsequent experiments.
\subsection{Computation Analysis}
The computational cost of CS-LB mainly consists of two parts: the cost incurred during the warm-up phase (Algorithm \ref{alg:Warm-up}), where the arm set is partitioned into clusters, and the per-round update cost during the subsequent online learning phase (Algorithm \ref{alg:CS-LB}).

Regarding the per-round update cost in the online learning phase, the analysis in the preceding sections has shown that each update occurs within a cluster whose rank does not exceed $l$, and the dominant cost arises from the SVD decomposition. Hence, the per-round update cost is $O(l^2d)$. Since $l$ can be tuned in advance according to the requirements of the environment, the setting $l^2<d$ is readily achievable, and thus CS-LB achieves a significant improvement in computational efficiency.

Regarding the computational cost of the warm-up phase, it mainly arises from the fact that each arm must perform an FD sketch update on each existing cluster to verify whether it can be added to that cluster. In the worst case, this incurs a total computational cost of $O(N\cdot N_{\mathcal{A}}^{l}\cdot l^2d)$, where $N_{\mathcal{A}}^{l}$ denotes the number of clusters. Since both the number of clusters and the rank of each cluster grow gradually during the construction process, the actual computational cost is smaller than this worst-case bound. On the other hand, the additional cost incurred during the warm-up phase is one-time only and does not appear in the subsequent online learning stage. When the horizon $T$ is sufficiently large, this one-time warm-up overhead becomes merely a constant term in the total computational cost. Thus, in terms of total computation, CS-LB achieves a significant efficiency gain.

\section{Numerical experiments}
In this section, we evaluate the performance of CS-LB and compare it with
UCB1\cite{UCB1}, OFUL\cite{OFUL}, SOFUL\cite{SOFUL} and CBSCFD\cite{CBSCFD}. We generated different synthetic datasets for different scenarios to facilitate comparison.
\subsection{Baseline}
\begin{itemize}
\item \textbf{UCB1.} Although the classic UCB1 method does not involve matrix computations and thus incurs low computational cost, UCB1 ignores the linear structure of the dataset. When the number of arms $N$ is large, its regret grows linearly with $N$. In contrast, our method exploits the correlations among arms via the linear structure within each cluster, and therefore outperforms traditional UCB1 methods on datasets with a strong linear structure.

\item \textbf{OFUL.} As a classic algorithm for solving linear bandits, OFUL enjoys a tight regret guarantee, but inevitably incurs a computational cost of $O(d^2)$ per round. Our method reduces the per-round computational cost to $O(l^2d)$ (excluding cluster partitioning), where $l$ can be adjusted as needed.

\item \textbf{SOFUL.} SOFUL is the first method to adopt a sketching strategy to reduce computational costs in linear bandits, with both low time and space overhead. However, it faces the risk of linear regret when the sketch size is set inappropriately. Our method retains full matrix information within each cluster without truncating singular values, thereby ensuring convergence.

\item \textbf{CBSCFD.} The CBSCFD algorithm improves upon SOFUL and achieves a tighter regret bound theoretically; however, it still fails to address the issue of linear regret arising from a heavy spectral tail.

\end{itemize} 

\subsection{Experimental Setup}
We set the additive random noise of the reward for each arm to follow a truncated Gaussian distribution with variance $R=1$, and the truncation interval is $[\mu_{\bm{x}}-3R,\mu_{\bm{x}}+3R]$, where $\mu_{\bm{x}}$ denotes the expected reward corresponding to arm $\bm{x}\in\mathcal{A}$. Each component of every arm vector is drawn independently from a uniform distribution over $(0,1)$. Each vector is then multiplied by a random integer coefficient and subsequently normalized to ensure that $\left \| \bm{x} \right \|\leq 1$. The ground-truth parameter $\bm{\theta^*}$ is generated in the same manner. Accordingly, we set $L=1$, $S=1$ and $\lambda=1$.

\textbf{Linear Correlation.} To demonstrate the advantages of the linear model, we incorporate linear correlation into the generation of our synthetic dataset. In generating the data, we introduce a linear correlation coefficient $m^{l}$, representing the number of linearly correlated vectors generated in succession. The arm set generation procedure incorporating $m^{l}$ is as follows: Assuming that $N$ arms are to be generated, these arms are divided into $N/m^{l}$ groups for generation. Each group consists of $m^{l}$ vectors that share the same direction but differ in their coefficients. The resulting dataset exhibits the following characteristics:
\begin{align}
\nonumber
&\alpha_{1,1}\bm{x}^{1},...,\alpha_{1,m^{l}}\bm{x}^{1},\\ \nonumber
&\alpha_{2,1}\bm{x}^{2},...,\alpha_{2,m^{l}}\bm{x}^{2},\\ \nonumber
&...\\ \nonumber
&\alpha_{N/m^{l},1}\bm{x}^{N/m^{l}},...,\alpha_{N/m^{l},m^{l}}\bm{x}^{N/m^{l}},
\end{align}
where $\alpha_{1,1}>0$ is a randomly generated real-valued coefficient. Finally, the dataset generated above is normalized by $\bm{x}/\left \| \bm{x} \right \|$.

\subsection{Regret Convergence Verification}
We set up two separate synthetic datasets as follows: (1)$N=60,d=50,l=5$, (2)$N=100,d=50,l=10$, where $N$ is the number of arms, $d$ is the dimension, and $l$ is the sketch size.

We compare the regret and running time (seconds) of our method against SOFUL, CBSCFD, and OFUL on these two simulated datasets. All reported results are averaged over five independent runs. Averaged experimental results are shown in Figures \ref{fig:regret1} and \ref{fig:time1}.

\begin{figure}[!ht]
\centering
\begin{subfigure}{.23\textwidth}
  \centering
  \includegraphics[width=\linewidth,height=1\textwidth]{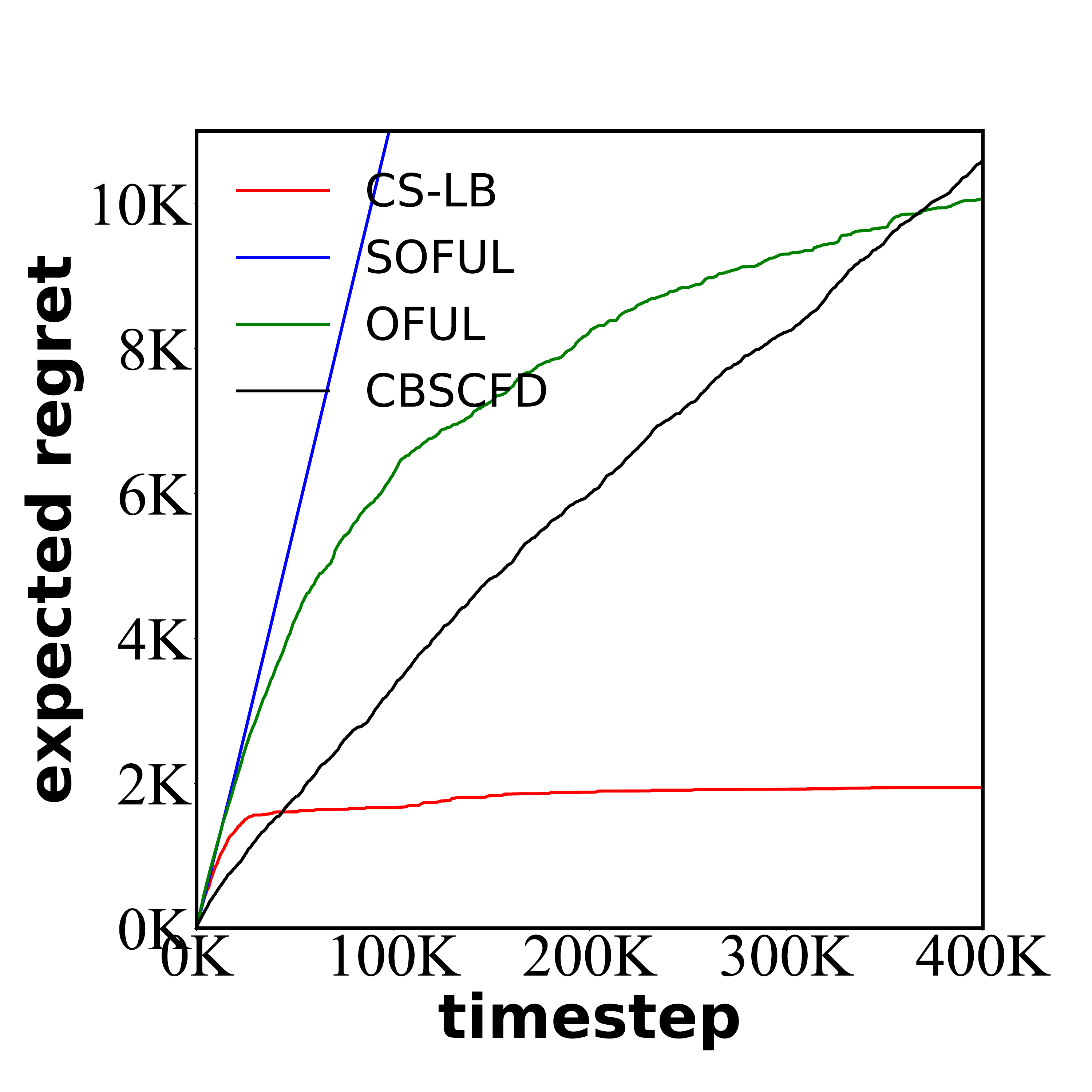}  
  \caption{$N=60,d=50,l=5$}
\end{subfigure}
\begin{subfigure}{.23\textwidth}
 \centering
  \includegraphics[width=\linewidth,height=1\textwidth]{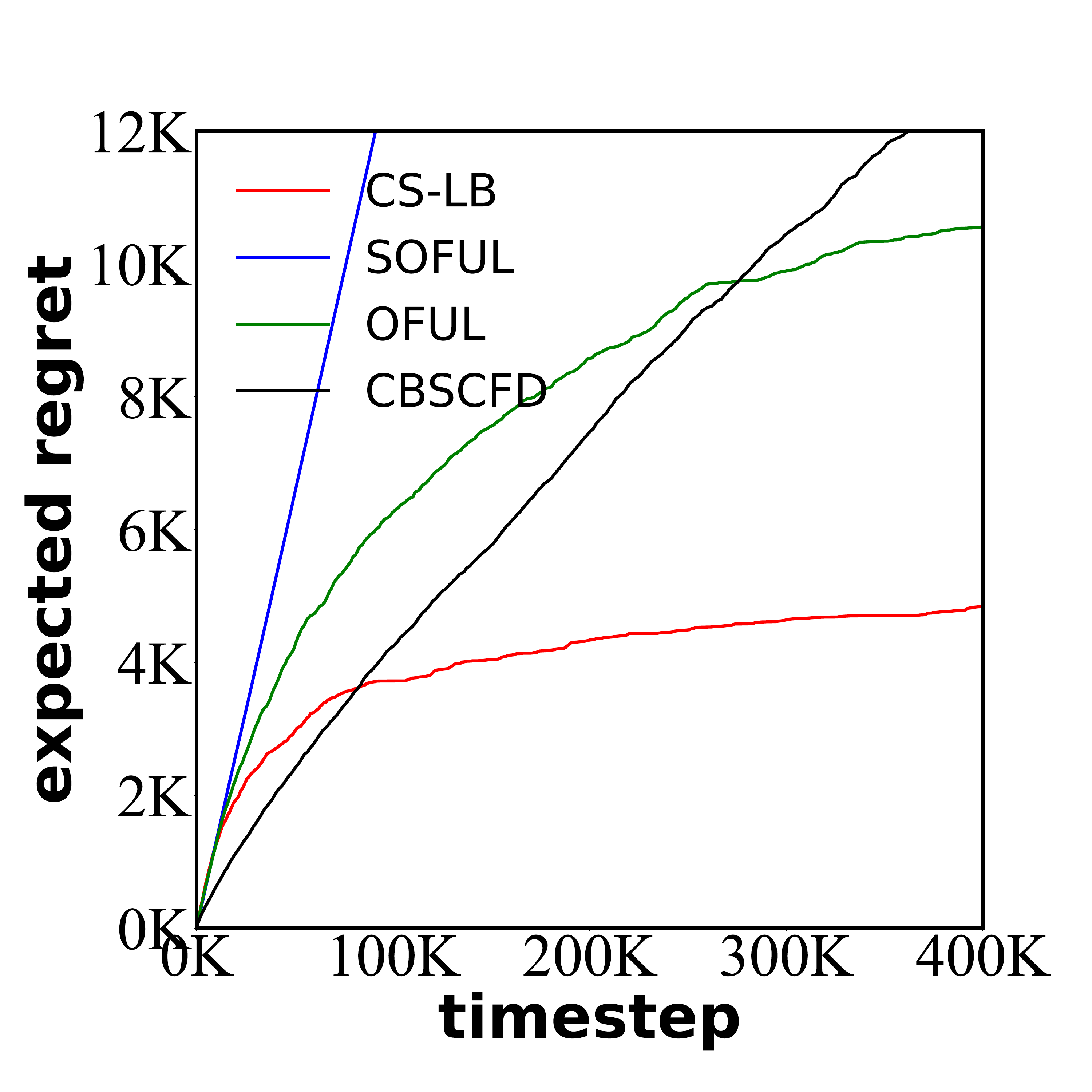}  
  \caption{$N=100,d=50,l=10$}
\end{subfigure}
\caption{Average Regret}
\label{fig:regret1}
\end{figure}

\begin{figure}[!ht]
\centering
\begin{subfigure}{.23\textwidth}
  \centering
  \includegraphics[width=\linewidth,height=1\textwidth]{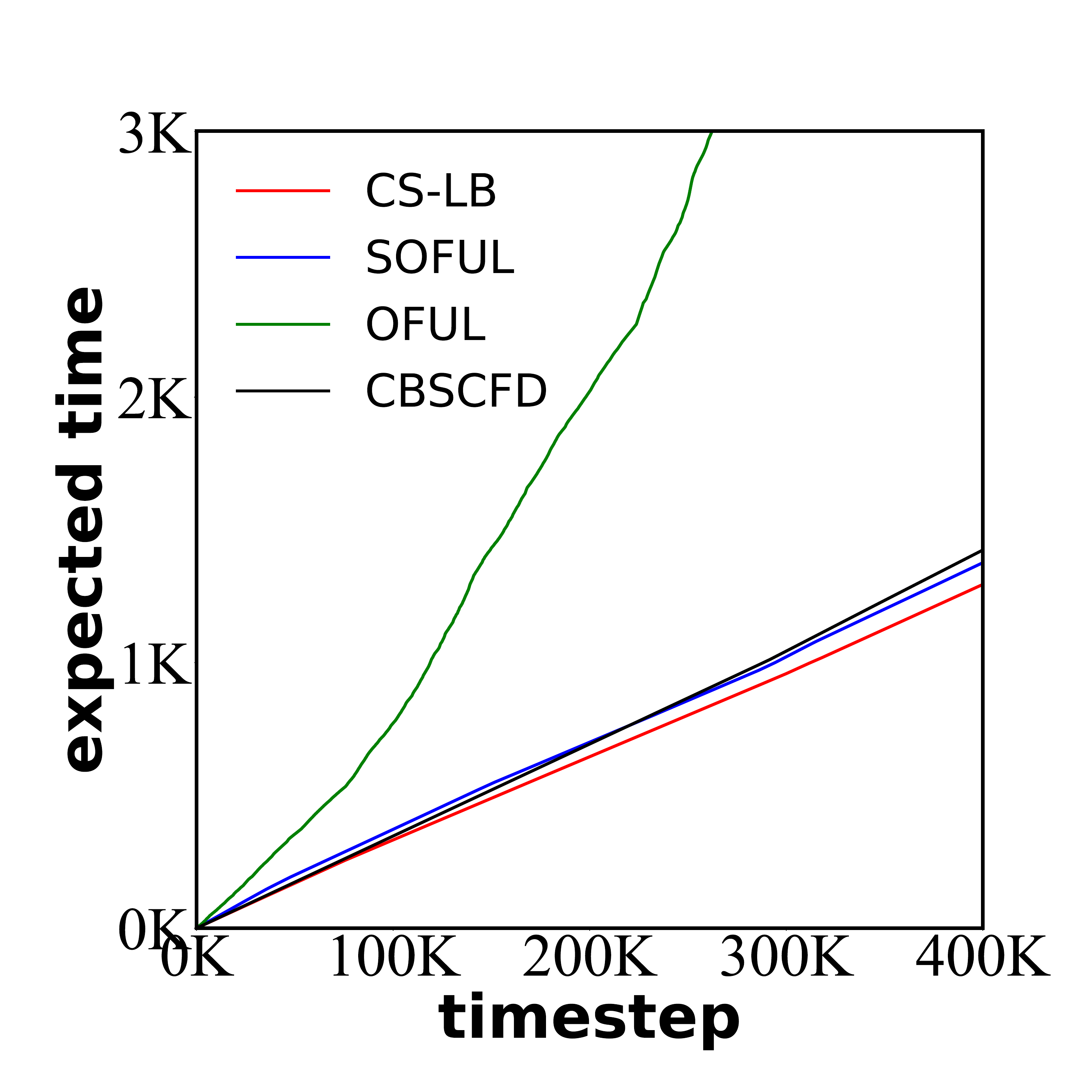}  
  \caption{$N=60,d=50,l=5$}
  \end{subfigure}
\begin{subfigure}{.23\textwidth}
  \centering
\includegraphics[width=\linewidth,height=1\textwidth]{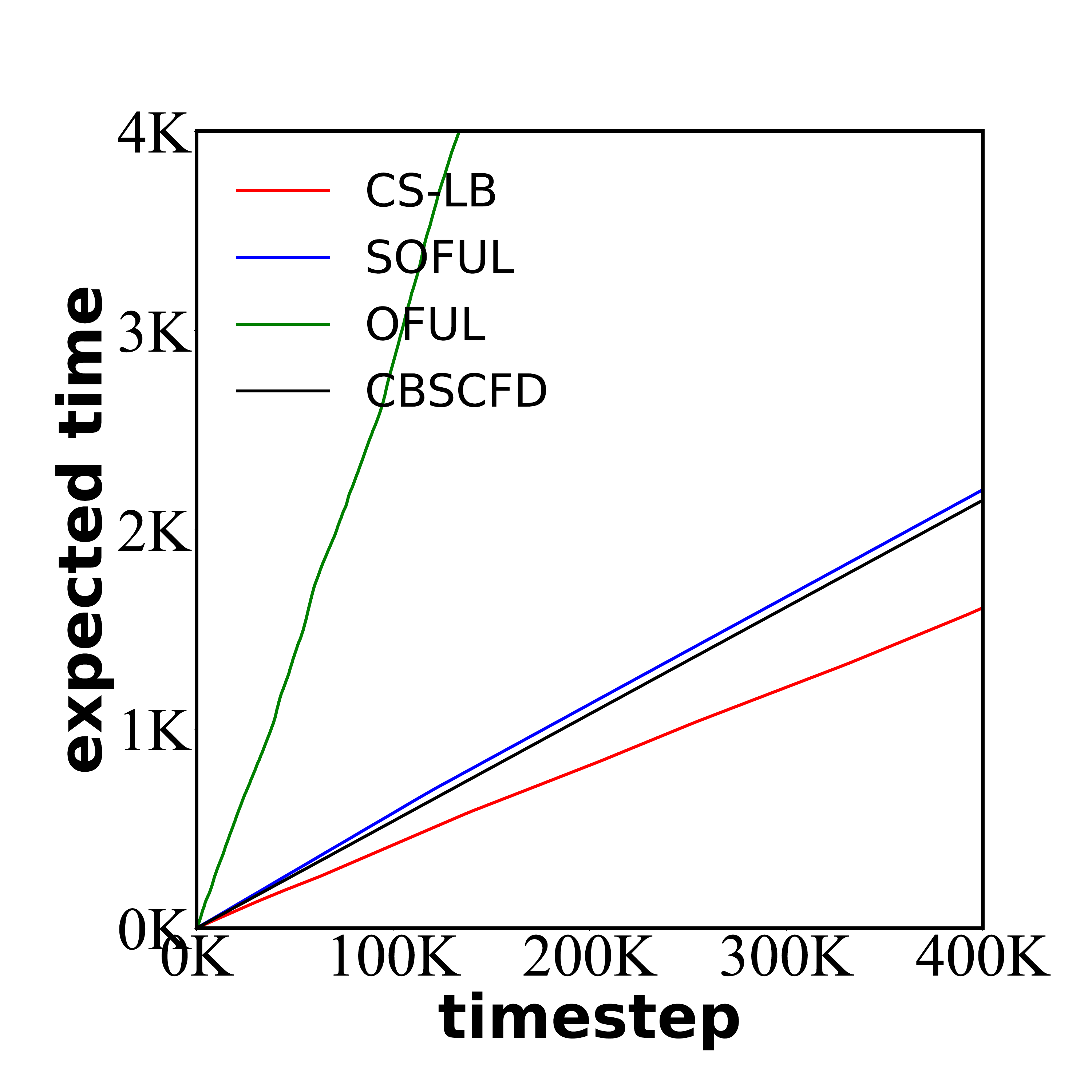}  
  \caption{$N=100,d=50,l=10$}
\end{subfigure}
\caption{Average Running Time}
\label{fig:time1}
\end{figure}

For Dataset (1), according to Theorem \ref{the:main}, the maximum number of clusters is $N/l=12$, so we have $l\sqrt{N_{\mathcal{A}}^{l}}<5\times \sqrt{12}<d=50$. Therefore, CS‑LB is expected to achieve the smallest regret among all compared algorithms. Meanwhile, SOFUL and CBSCFD suffer from non‑convergence due to the chosen sketch size. The same reasoning applies to Dataset (2). The experimental results are consistent with our theoretical analysis.

We also observe that CS‑LB achieves slightly lower running time than SOFUL and CBSCFD. We attribute this to the fact that CS‑LB restricts the search for the arm to be played to the current cluster, whose size is significantly smaller than the entire arm set of size $N$. Consequently, on these datasets, CS‑LB not only reduces computational cost but also lowers the search overhead.

\subsection{Comparing with UCB1}
For the fixed-arm-set setting, a straightforward approach is the classic UCB1 algorithm, which incurs no matrix operations and thus achieves the fastest running time. However, UCB1 fails to exploit the internal structure of the linear model, causing its regret to deteriorate as the number of arms grows. To demonstrate that the linear bandit problem is not trivially solvable by a simple UCB1 strategy, we generate the following two synthetic datasets according to the experimental setup for linear correlation described above: (3) $N=100,d=50,l=10,m^l=2$, (4) $N=200,d=50,l=5,m^l=10$.

Based on the definition of linear correlation coefficient $m^l$ above, Dataset (4) has higher linear correlation and a smaller rank than Dataset (3). All reported results are averaged over five independent runs. Averaged experimental results are shown in Figures \ref{fig:regret2} and \ref{fig:time2}.

\begin{figure}[!ht]
\centering
\begin{subfigure}{.23\textwidth}
  \centering
\includegraphics[width=\linewidth,height=1\textwidth]{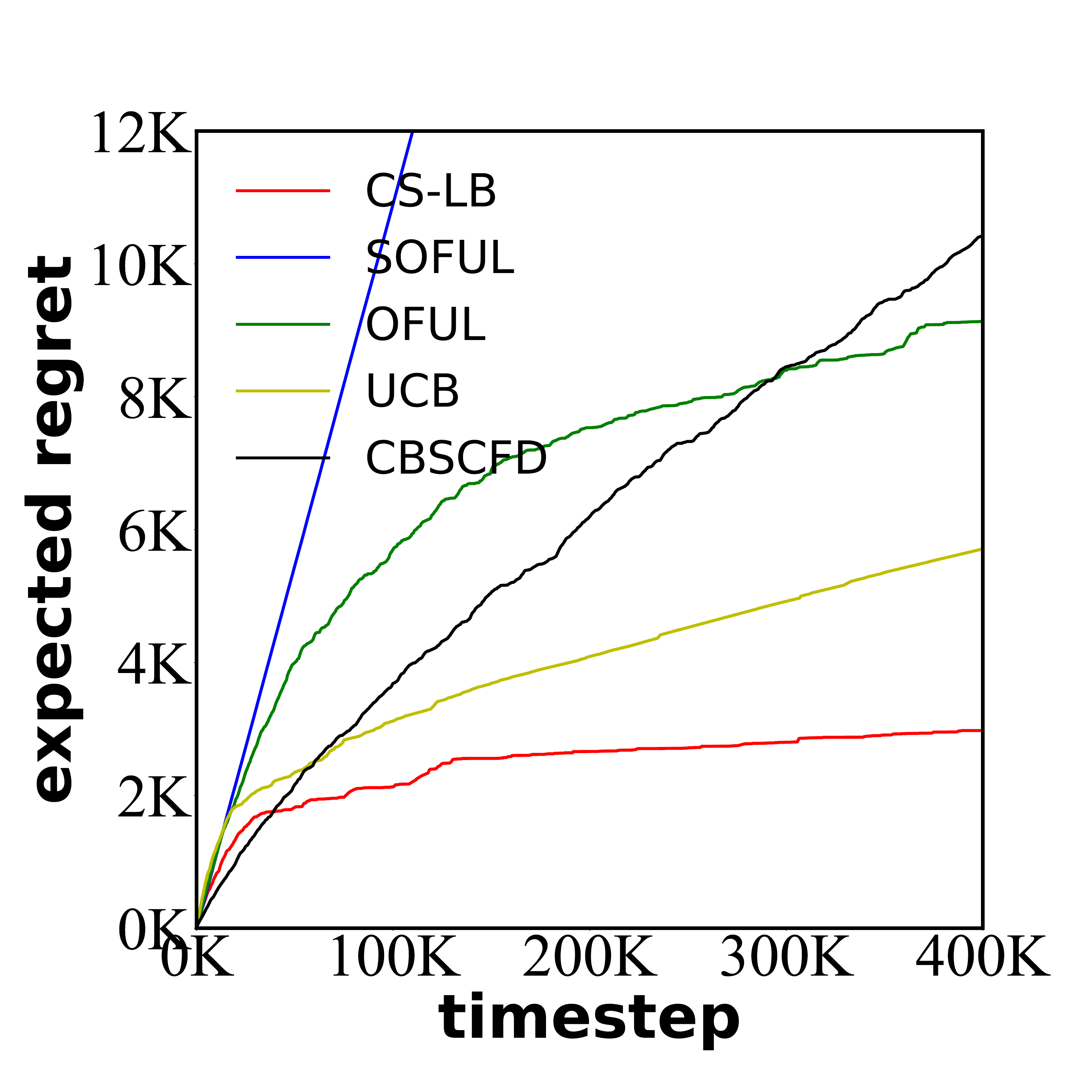}  
  \caption{$m^l=2$}
\end{subfigure}
\begin{subfigure}{.23\textwidth}
  \centering
\includegraphics[width=\linewidth,height=1\textwidth]{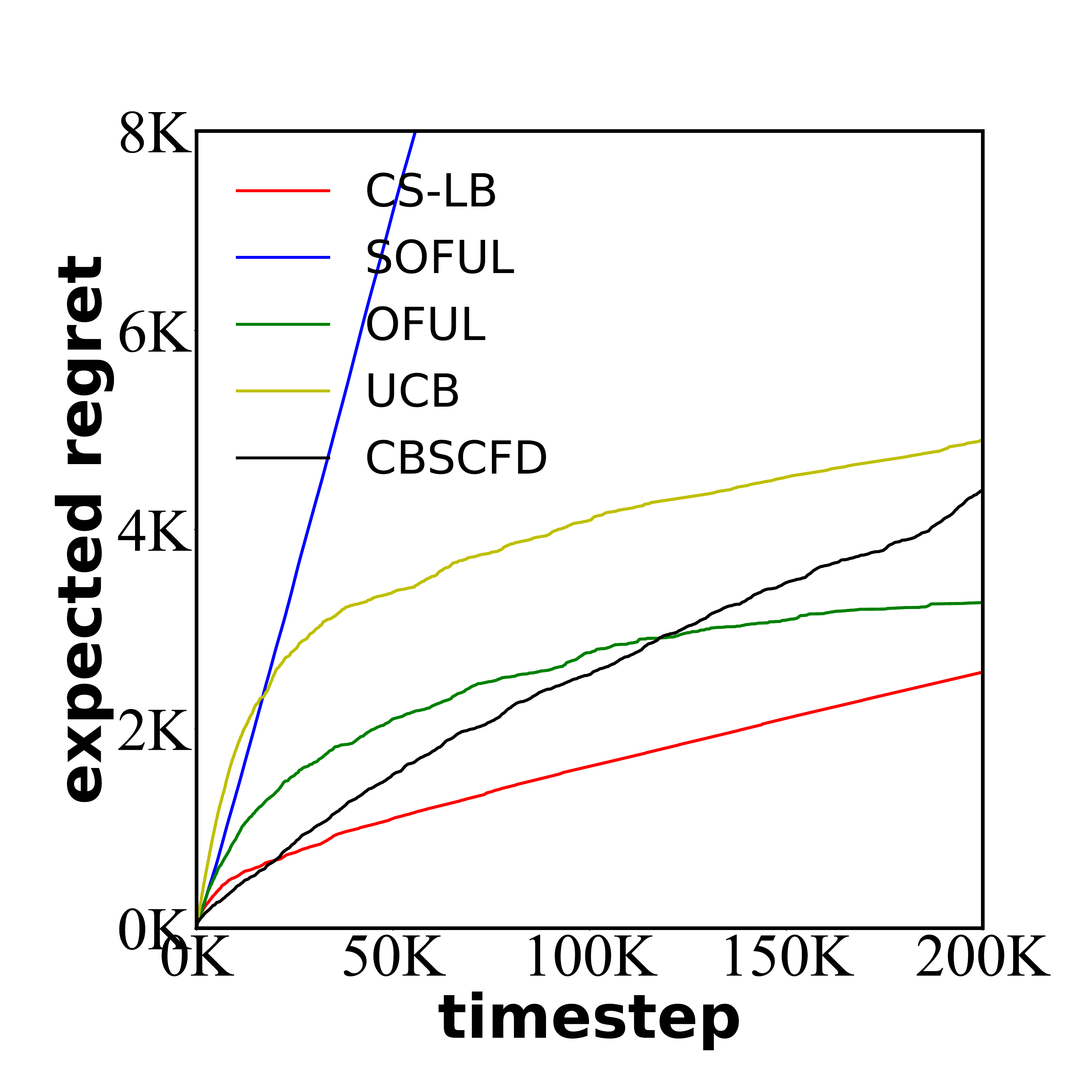}  
  \caption{$m^l=10$}
\end{subfigure}
\caption{Average Regret}
\label{fig:regret2}
\end{figure}

\begin{figure}[!ht]
\centering
\begin{subfigure}{.23\textwidth}
  \centering
\includegraphics[width=\linewidth,height=1\textwidth]{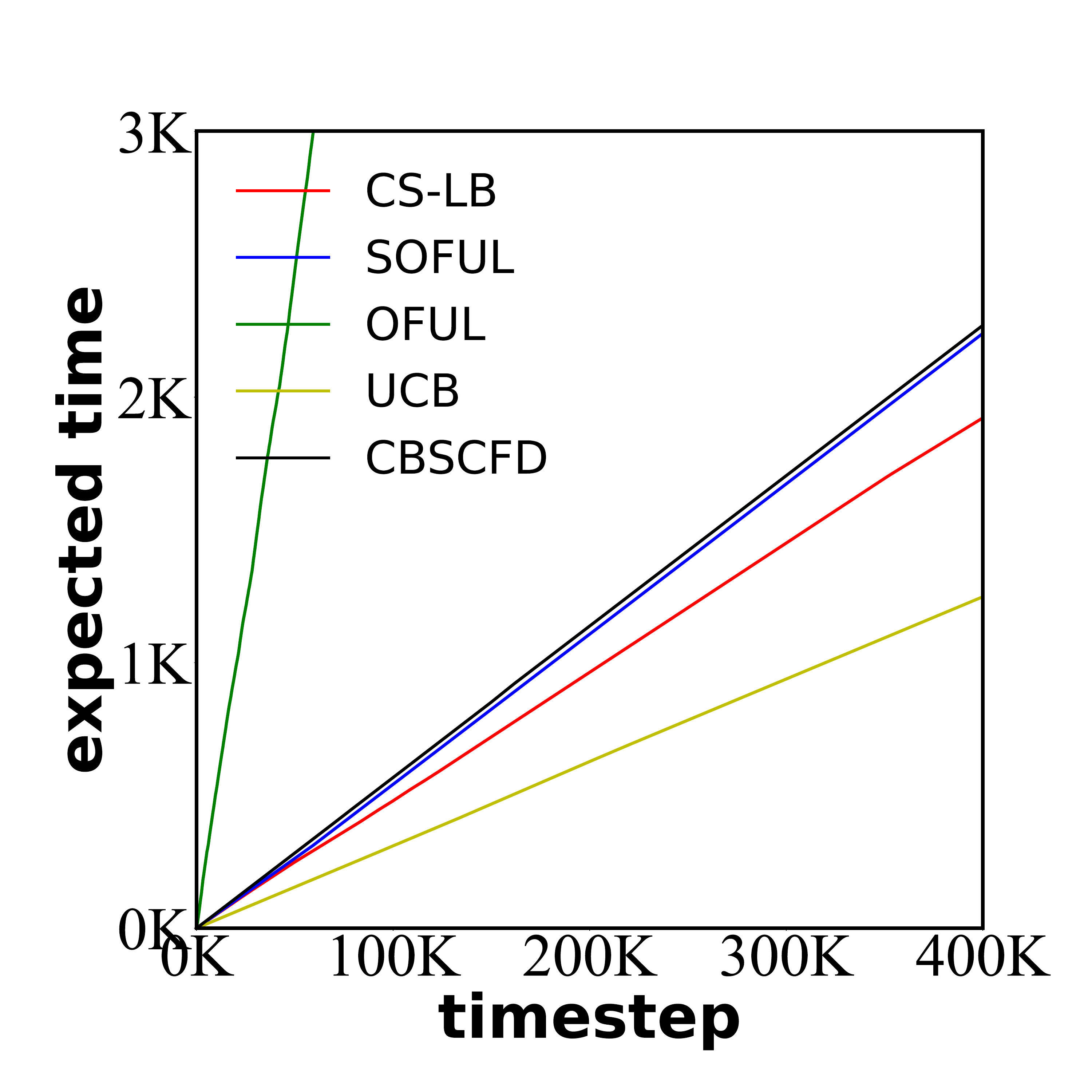}  
  \caption{$m^l=2$}
  \end{subfigure}
\begin{subfigure}{.23\textwidth}
  \centering
\includegraphics[width=\linewidth,height=1\textwidth]{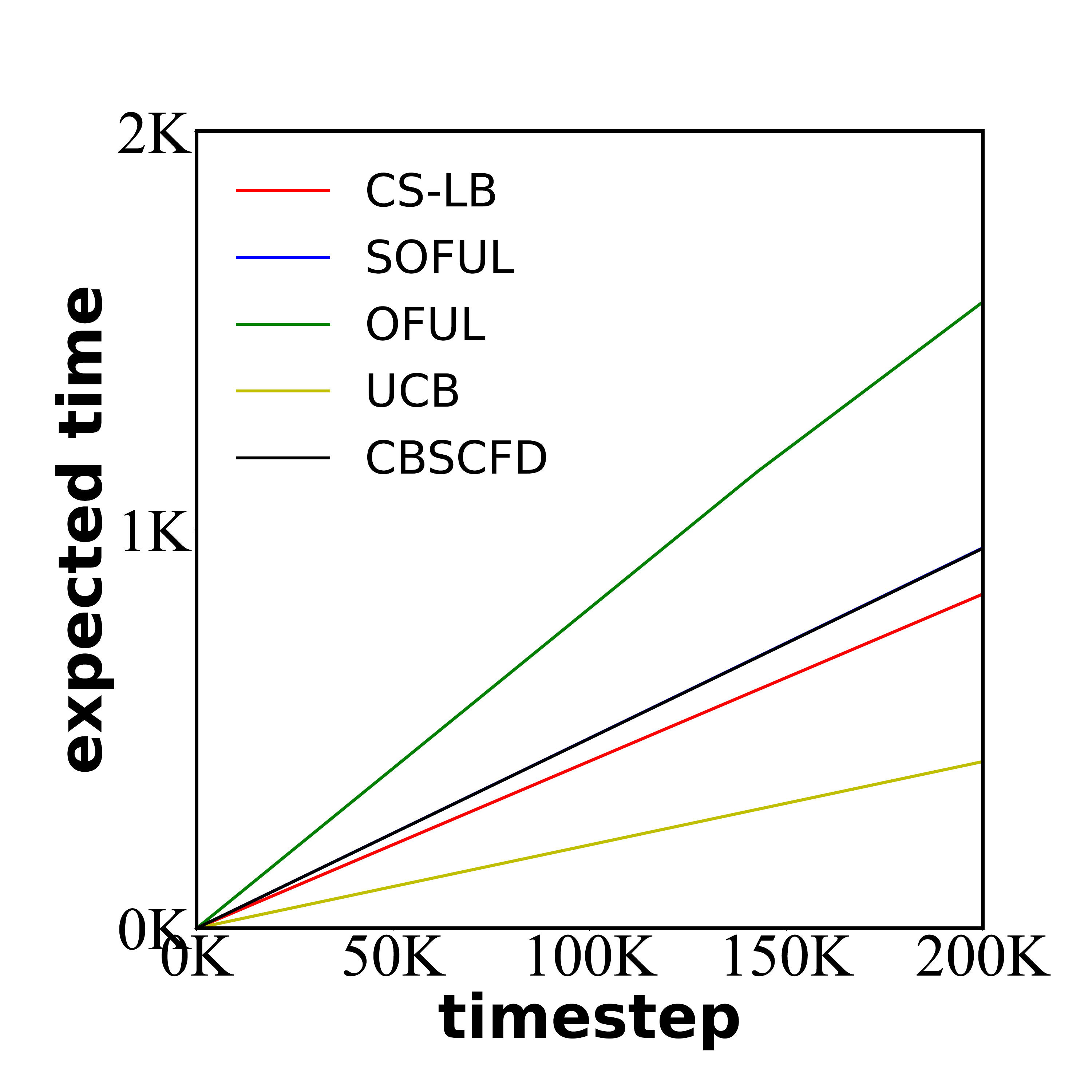}  
  \caption{$m^l=10$}
\end{subfigure}
\caption{Average Running Time}
\label{fig:time2}
\end{figure}

Experimental results demonstrate that, while UCB1 is computationally cheaper, OFUL and CS-LB achieve superior convergence rates as the arm set expands, by effectively exploiting the linear structure. Hence, sketching strategies retain considerable practical value for fixed-arm-set linear bandits, particularly when the arm set is large.

\section{Conclusion}
This paper addresses the non-convergence and regret inflation caused by cumulative spectral errors in sketching-based linear bandits. We propose Cluster Sketch Linear Bandits (CS‑LB), the first sketching strategy that achieves convergence under a fixed sketch size, via clustering that partitions the action set, maintains lossless per‑cluster sketches, and uses a sentinel to select the optimistic cluster each round. We prove that our method achieves sublinear regret with guaranteed convergence, and reduce per‑round computation from $O(d^2)$ to $O(ld)$. Experimental evaluations on synthetic data show that CS‑LB converges reliably, lowers computational overhead, and achieves superior regret performance.

We acknowledge that our current method still has limitations. Since the warm-up phase of the algorithm relies on reading the entire dataset, CS‑LB is not yet applicable to more general linear contextual bandits with a varying arm set. We plan to extend our approach in future work.

\newpage
\bibliographystyle{unsrt}  
\bibliography{references}

@inproceedings{OFUL,
  title={Improved Algorithms for Linear Stochastic Bandits},
  author={Yasin Abbasi-Yadkori and D{\'a}vid P{\'a}l and Csaba Szepesvari},
  booktitle={Neural Information Processing Systems},
  year={2011}
}

@article{UCB1,
author = {Auer, Peter and Cesa-Bianchi, Nicol\`{o} and Fischer, Paul},
title = {Finite-time Analysis of the Multiarmed Bandit Problem},
year = {2002},
issue_date = {May-June 2002},
publisher = {Kluwer Academic Publishers},
address = {USA},
volume = {47},
number = {2–3},
issn = {0885-6125},
url = {https://doi.org/10.1023/A:1013689704352},
doi = {10.1023/A:1013689704352},
journal = {Mach. Learn.},
month = may,
pages = {235–256},
numpages = {22}
}

@InProceedings{SOFUL,
  title = 	 {Efficient Linear Bandits through Matrix Sketching},
  author =       {Kuzborskij, Ilja and Cella, Leonardo and Cesa-Bianchi, Nicol\`{o}},
  booktitle = 	 {Proceedings of the Twenty-Second International Conference on Artificial Intelligence and Statistics},
  pages = 	 {177--185},
  year = 	 {2019},
  editor = 	 {Chaudhuri, Kamalika and Sugiyama, Masashi},
  volume = 	 {89},
  series = 	 {Proceedings of Machine Learning Research},
  month = 	 {16--18 Apr},
  publisher =    {PMLR},
  url = 	 {https://proceedings.mlr.press/v89/kuzborskij19a.html},
}

@inproceedings{CBSCFD,
author = {Chen, Cheng and Luo, Luo and Zhang, Weinan and Yu, Yong and Lian, Yijiang},
title = {Efficient and robust high-dimensional linear contextual bandits},
year = {2021},
isbn = {9780999241165},
articleno = {588},
numpages = {7},
location = {Yokohama, Yokohama, Japan},
series = {IJCAI'20}
}

@inproceedings{645528.657779,
author = {Abe, Naoki and Long, Philip M.},
title = {Associative Reinforcement Learning using Linear Probabilistic Concepts},
year = {1999},
isbn = {1558606122},
publisher = {Morgan Kaufmann Publishers Inc.},
address = {San Francisco, CA, USA},
booktitle = {Proceedings of the Sixteenth International Conference on Machine Learning},
pages = {3–11},
numpages = {9},
series = {ICML '99}
}

@article{944919.944941,
author = {Auer, Peter},
title = {Using confidence bounds for exploitation-exploration trade-offs},
year = {2003},
issue_date = {3/1/2003},
publisher = {JMLR.org},
volume = {3},
number = {null},
issn = {1532-4435},
journal = {J. Mach. Learn. Res.},
month = mar,
pages = {397–422},
numpages = {26}
}

@inproceedings{Dani2008StochasticLO,
  title={Stochastic Linear Optimization under Bandit Feedback},
  author={Varsha Dani and Thomas P. Hayes and Sham M. Kakade},
  booktitle={Annual Conference Computational Learning Theory},
  year={2008},
  url={https://api.semanticscholar.org/CorpusID:9134969}
}

@article{LinearlyParameterized,
author = {Rusmevichientong, Paat and Tsitsiklis, John N.},
title = {Linearly Parameterized Bandits},
year = {2010},
issue_date = {May 2010},
publisher = {INFORMS},
address = {Linthicum, MD, USA},
volume = {35},
number = {2},
issn = {0364-765X},
url = {https://doi.org/10.1287/moor.1100.0446},
doi = {10.1287/moor.1100.0446},
journal = {Math. Oper. Res.},
month = may,
pages = {395–411},
numpages = {17}
}

@inproceedings{ContextualBanditsLinearPayoff,
  title={Contextual Bandits with Linear Payoff Functions},
  author={Wei Chu and Lihong Li and L. Reyzin and Robert E. Schapire},
  booktitle={International Conference on Artificial Intelligence and Statistics},
  year={2011},
  url={https://api.semanticscholar.org/CorpusID:1452971}
}

@inproceedings{LinUCB,
author = {Li, Lihong and Chu, Wei and Langford, John and Schapire, Robert E.},
title = {A contextual-bandit approach to personalized news article recommendation},
year = {2010},
isbn = {9781605587998},
publisher = {Association for Computing Machinery},
address = {New York, NY, USA},
url = {https://doi.org/10.1145/1772690.1772758},
doi = {10.1145/1772690.1772758},
booktitle = {Proceedings of the 19th International Conference on World Wide Web},
pages = {661–670},
numpages = {10},
location = {Raleigh, North Carolina, USA},
series = {WWW '10}
}

@inproceedings{Thompsonsamplingbandits,
author = {Agrawal, Shipra and Goyal, Navin},
title = {Thompson sampling for contextual bandits with linear payoffs},
year = {2013},
publisher = {JMLR.org},
booktitle = {Proceedings of the 30th International Conference on International Conference on Machine Learning - Volume 28},
pages = {III–1220–III–1228},
location = {Atlanta, GA, USA},
series = {ICML'13}
}

@article{Thompson1933ONTL,
  title={ON THE LIKELIHOOD THAT ONE UNKNOWN PROBABILITY EXCEEDS ANOTHER IN VIEW OF THE EVIDENCE OF TWO SAMPLES},
  author={William R. Thompson},
  journal={Biometrika},
  year={1933},
  volume={25},
  pages={285-294},
  url={https://api.semanticscholar.org/CorpusID:120462794}
}

@inproceedings{Abeille2016LinearTS,
  title={Linear Thompson Sampling Revisited},
  author={Marc Abeille and Alessandro Lazaric},
  booktitle={International Conference on Artificial Intelligence and Statistics},
  year={2016},
  url={https://api.semanticscholar.org/CorpusID:8385089}
}

@inproceedings{10.1145/3055399.3055431,
author = {Song, Zhao and Woodruff, David P. and Zhong, Peilin},
title = {Low rank approximation with entrywise l1-norm error},
year = {2017},
isbn = {9781450345286},
publisher = {Association for Computing Machinery},
address = {New York, NY, USA},
url = {https://doi.org/10.1145/3055399.3055431},
doi = {10.1145/3055399.3055431},
booktitle = {Proceedings of the 49th Annual ACM SIGACT Symposium on Theory of Computing},
pages = {688–701},
numpages = {14},
location = {Montreal, Canada},
series = {STOC 2017}
}

@InProceedings{pmlr-v80-andoni18a,
  title = 	 {Subspace Embedding and Linear Regression with Orlicz Norm},
  author =       {Andoni, Alexandr and Lin, Chengyu and Sheng, Ying and Zhong, Peilin and Zhong, Ruiqi},
  booktitle = 	 {Proceedings of the 35th International Conference on Machine Learning},
  pages = 	 {224--233},
  year = 	 {2018},
  editor = 	 {Dy, Jennifer and Krause, Andreas},
  volume = 	 {80},
  series = 	 {Proceedings of Machine Learning Research},
  month = 	 {10--15 Jul},
  publisher =    {PMLR},
  url = 	 {https://proceedings.mlr.press/v80/andoni18a.html},
}

@misc{BeyondJohnson,
      title={Beyond Johnson-Lindenstrauss: Uniform Bounds for Sketched Bilinear Forms}, 
      author={Rohan Deb and Qiaobo Li and Mayank Shrivastava and Arindam Banerjee},
      year={2025},
      eprint={2509.21847},
      archivePrefix={arXiv},
      primaryClass={cs.LG},
      url={https://arxiv.org/abs/2509.21847}, 
}

@article{Yu_Lyu_King_2017, 
title={CBRAP: Contextual Bandits with RAndom Projection}, 
volume={31}, url={https://ojs.aaai.org/index.php/AAAI/article/view/10888}, 
DOI={10.1609/aaai.v31i1.10888}, 
number={1}, 
journal={Proceedings of the AAAI Conference on Artificial Intelligence}, 
author={Yu, Xiaotian and Lyu, Michael R. and King, Irwin}, 
year={2017}, 
month={Feb.} 
}

@inproceedings{DBSLinUCB,
  title={Revisiting Matrix Sketching in Linear Bandits: Achieving Sublinear Regret via Dyadic Block Sketching},
  author={Dongxie Wen and Hanyan Yin and Xiao Zhang and Peng Zhao and Lijun Zhang and Zhewei Wei},
  year={2024},
  url={https://api.semanticscholar.org/CorpusID:273346688}
}

@article{FD,
author = {Ghashami, Mina and Liberty, Edo and Phillips, Jeff M. and Woodruff, David P.},
title = {Frequent Directions: Simple and Deterministic Matrix Sketching},
year = {2016},
issue_date = {2016},
publisher = {Society for Industrial and Applied Mathematics},
address = {USA},
volume = {45},
number = {5},
issn = {0097-5397},
url = {https://doi.org/10.1137/15M1009718},
doi = {10.1137/15M1009718},
journal = {SIAM J. Comput.},
month = jan,
pages = {1762–1792},
numpages = {31}
}

@inproceedings{10.1145/2505515.2514700,
author = {Tang, Liang and Rosales, Romer and Singh, Ajit and Agarwal, Deepak},
title = {Automatic ad format selection via contextual bandits},
year = {2013},
isbn = {9781450322638},
publisher = {Association for Computing Machinery},
address = {New York, NY, USA},
url = {https://doi.org/10.1145/2505515.2514700},
doi = {10.1145/2505515.2514700},
booktitle = {Proceedings of the 22nd ACM International Conference on Information \& Knowledge Management},
pages = {1587–1594},
numpages = {8},
location = {San Francisco, California, USA},
series = {CIKM '13}
}

@inproceedings{Tewari2017FromAT,
  title={From Ads to Interventions: Contextual Bandits in Mobile Health},
  author={Ambuj Tewari and Susan A. Murphy},
  booktitle={Mobile Health - Sensors, Analytic Methods, and Applications},
  year={2017},
  url={https://api.semanticscholar.org/CorpusID:18778220}
}


\newpage
\appendix
\section{Supplementary Proof}
\subsection{Notation in Proof}
To facilitate a formal proof, we provide the following definitions.

Let $\mathcal{C}^{all}\longleftarrow CSW(\mathcal{A},l)$ and $\bm{x}_t$ is arm selected in round $t$. We use $\mathcal{C}^{\bm{x}}\in \mathcal{C}^{all}$ to denote the cluster that $\bm{x}$ belongs to after running Algorithm \ref{alg:Warm-up}.

When running CS-LB, for each cluster $\mathcal{C}\in \mathcal{C}^{all}$, we define:
\begin{align}
&\bm{X}_{\mathcal{C},t}=[\bm{x}_1,...,\bm{x}_t]_{x_i\in \mathcal{C},1\leq i\leq t}^{\top},\\
&\bm{V}_{\mathcal{C},t}=\lambda\bm{I}+\sum_{s=1,\bm{x}_s\in \mathcal{C}}^{t}\bm{x}_s\bm{x}_s^{\top},\\
&\bar{\bm{V}}_{\mathcal{C},t}=\lambda\bm{I}+\bm{S}_{\mathcal{C},t}^{\top}\bm{S}_{\mathcal{C},t},\\
&\bm{\beta}_{\mathcal{C},t}=\sum_{s=1,\bm{x}_s\in \mathcal{C}}^{t}y_{s}\bm{x}_s,\\
&\bm{\hat{\theta}}_{\mathcal{C},t}=\bar{\bm{V}}_{\mathcal{C},t}^{-1}\bm{\beta}_{\mathcal{C},t}.
\end{align}
We use $T_{\mathcal{C},t}$ to denote the number of times each cluster has been selected up to round $t$, i.e.
\begin{align}
    T_{\mathcal{C},t}=\sum_{s=1}^{t}\mathbb{I}\{\bm{x}_s\in \mathcal{C}\}=\sum_{s=1,\bm{x}_s\in \mathcal{C}}^{t}1.
\end{align}
It is obvious that $\sum_{\mathcal{C}\in\mathcal{C}^{all}}T_{\mathcal{C},t}=t$.

We use
\begin{align}
&\bm{V}_t=\lambda\bm{I}+\sum_{s=1}^{t}\bm{x}_s\bm{x}_s^{\top},\\
&\bm{\beta}_t=\sum_{s=1}^{t}y_{s}\bm{x}_s,\\
&\bm{\theta}_t=\bm{V}_t^{-1}\bm{\beta}_t
\end{align}
to denote the statistical information maintained during the execution of OFUL.

\subsection{Technical Lemmas}
This section collects the key, previously established properties and lemmas that we will use in our proofs.
\begin{lemma}
\label{tech0}
For any positive semi-definite matrix $\bm{A}$ and any vector $\bm{x},\bm{y}$, we have $|\bm{x}^{\top}\bm{y}|\leq \left \|\bm{x}\right \|_{\bm{A}}\left \|\bm{y}\right \|_{\bm{A}^{-1}}$.
\end{lemma}
\begin{proof}
\begin{align}
|\bm{x}^{\top}\bm{y}|
&=\bm{x}^{\top}\bm{A}^{\frac{1}{2}}\bm{A}^{-\frac{1}{2}}\bm{y}\\
&\leq \left \|\bm{A}^{\frac{1}{2}}\bm{x} \right \|_2\left \|\bm{A}^{-\frac{1}{2}}\bm{y} \right \|_2 \text{ (Cauchy-Schwartz inequality)}\\
&=\left \|\bm{x}\right \|_{\bm{A}}\left \|\bm{y}\right \|_{\bm{A}^{-1}}.
\end{align}
\end{proof}

\begin{lemma}
\label{tech1}
(Theorem 2 in \cite{OFUL}) For any $\delta>0$, with probability at least $1-\delta$, for all $t\geq 0$, $\bm{\theta}^*$ lies in the set
\begin{align}
    \left\{\bm{\theta}\in \mathbb{R}^d: \left \|\hat{\bm{\theta}}_t- \bm{\theta} \right \|_{\bm{V}_t} \leq R\sqrt{2\log\left(\frac{det(\bm{V}_t)^{1/2}det(\lambda\bm{I})^{-1/2}}{\delta} \right)}+\sqrt{\lambda} S \right\}
\end{align}
\end{lemma}

\begin{lemma}
\label{tech2}
According to \cite{OFUL} and \cite{SOFUL}, with sketch size $l$, we have
\begin{align}
\nonumber
\sum_{t=1}^{T}\min\{1,\left \|\bm{x}_t \right \|_{\bm{V}_t^{-1}}^2\} 
&\leq 2\ln\frac{det(\bm{V}_t)}{det(\lambda\bm{I})}\\
\nonumber
&\leq 2d\ln\left (1+\varepsilon_l \right )+2l\ln\left (1+\frac{tL^2}{l\lambda} \right ).
\end{align}
where $\varepsilon_l=\min_{k=0,...,l-1}\frac{\lambda_{d-k}+\lambda_{d-k+1}+...+\lambda_{d}}{\lambda(l-k)}$ and $\lambda_1>...>\lambda_d$ are the eigenvalues of the correlation matrix $\bm{V}_t=\lambda\bm{I}+\sum_{s=1}^{t}\bm{x}\bm{x}^{\top}$, $\varepsilon_l$ is upper bounded by the spectral tail of the covariance matrix under $l$ sketch size.
\end{lemma}
\subsection{Proof of Lemma \ref{lem:main1}}
We first derive the following lemma from Lemma 10 in \cite{OFUL}:
\begin{lemma}
\label{lem:sketch trace}
(Sketch-Trace Inequality) For a given sketch size $l$ and any cluster $\mathcal{C}\in \mathcal{C}^{all}$, we have
\begin{align}
det(\bm{V}_{\mathcal{C},t})=det(\bar{\bm{V}}_{\mathcal{C},t})\leq \left(\lambda+tL^2/l \right)^{l}\cdot \lambda^{d-l}.
\end{align}
\end{lemma}
\begin{proof}
First, according to Proposition 3 in \cite{SOFUL}, no singular value truncation is performed during the sketch update process of CS-LB. Therefore, we have
\begin{align}
    \bm{S}_{\mathcal{C},t}^{\top}\bm{S}_{\mathcal{C},t}=\bm{X}_{\mathcal{C},t}^{\top}\bm{X}_{\mathcal{C},t}=\sum_{s=1,\bm{x}_s\in \mathcal{C}}^{t}\bm{x}_s\bm{x}_s^{\top}.
\end{align}
Then we have
\[
\bm{V}_{\mathcal{C},t}=\bar{\bm{V}}_{\mathcal{C},t}.
\]
and according to the process of Algorithm \ref{alg:Warm-up}, the rank of $\bm{S}_{\mathcal{C},t}^{\top}\bm{S}_{\mathcal{C},t}$ is less than $l$.

Let $\alpha_1,...,\alpha_d$ be the eigenvalues of $\bm{V}_{\mathcal{C},t}$ and $\widetilde{\alpha}_1,...,\widetilde{\alpha}_d$ be the eigenvalues of $\bm{S}_{\mathcal{C},t}^{\top}\bm{S}_{\mathcal{C},t}$. Because $rank(\bm{S}_{\mathcal{C},t}^{\top}\bm{S}_{\mathcal{C},t})\leq l$, then, by the properties of singular value decomposition, we obtain
\[
\alpha_1,...,\alpha_l=\lambda+\widetilde{\alpha}_1,...,\lambda+\widetilde{\alpha}_d,
\]
and
\[
\alpha_{l+1}=...=\alpha_{d}=\lambda.
\]
Then $det(\bm{V}_{\mathcal{C},t})=\prod_{i=1}^{l}\alpha_i \cdot \lambda^{d-l}$. We have
\begin{align}
    \prod_{i=1}^{l}\alpha_i
    &\leq \left(\frac{\sum_{i=1}^{l}\alpha_i}{l} \right)^{l}\\
&=\left(\frac{l\lambda+\sum_{i=1}^{l}\widetilde{\alpha}_i}{l} \right)^{l}\\
&=\left(\frac{l\lambda+trace(\bm{S}_{\mathcal{C},t}^{\top}\bm{S}_{\mathcal{C},t})}{l} \right)^{l}\\
&=\left(\frac{l\lambda+trace(\bm{X}_{\mathcal{C},t}^{\top}\bm{X}_{\mathcal{C},t})}{l} \right)^{l}.
\end{align}
The first inequality is because AM-GM inequality. According to Lemma 10 in \cite{OFUL}, we have $trace(\bm{X}_{\mathcal{C},t}^{\top}\bm{X}_{\mathcal{C},t})\leq \sum_{s=1,\bm{x}_s\in \mathcal{C}}^{t}L^2\leq T_{\mathcal{C},t}L^2\leq tL^2$.
Substituting the inequality yields
\begin{align}
det(\bm{V}_{\mathcal{C},t})=det(\bar{\bm{V}}_{\mathcal{C},t})\leq \left(\lambda+tL^2/l \right)^{l}\cdot \lambda^{d-l}.
\end{align}
\end{proof}
Combining Lemma \ref{tech1} and Lemma \ref{lem:sketch trace}, because $\bm{V}_{\mathcal{C},t}=\bar{\bm{V}}_{\mathcal{C},t}$, we obtain that for each cluster $\mathcal{C}$, with probability at least $1-\delta'$, the following holds:
\begin{align}
    \left \|\bm{\hat{\theta}}_{\mathcal{C},t}- \bm{\theta}^* \right \|_{\bm{V}_{\mathcal{C},t}} 
    &\leq R\sqrt{2\log\left(\frac{det(\bm{V}_{\mathcal{C},t})^{1/2}det(\lambda\bm{I})^{-1/2}}{\delta'} \right)}+\sqrt{\lambda} S\\
    &\leq R\sqrt{2\log\left(\frac{\left(\lambda^{d-l}\left(\lambda+tL^2/l \right)^{l}\right)^{1/2}\lambda^{-d/2}}{\delta'} \right)}+\sqrt{\lambda} S\\
    &= R\sqrt{2l\log(1+\frac{tL^2}{l\lambda})+2\log\frac{1}{\delta'}}+S\sqrt{\lambda}.
\end{align}
Let $\delta'=\frac{\delta}{N_{\mathcal{A}}^l}$, then for any $\mathcal{C}\in \mathcal{C}^{all}$, with probability at least $1-\delta$, we have
\[
\left \|\bm{\hat{\theta}}_{\mathcal{C},t}- \bm{\theta}^* \right \|_{\bar{\bm{V}}_{\mathcal{C},t}}=\left \|\bm{\hat{\theta}}_{\mathcal{C},t}- \bm{\theta}^* \right \|_{\bm{V}_{\mathcal{C},t}}\leq R\sqrt{2l\log(1+\frac{tL^2}{l\lambda})+2\log\frac{N_{\mathcal{A}}^l}{\delta}}+S\sqrt{\lambda}. 
\]
Because $N_{\mathcal{A}}^l\leq N/r$, we finish the proof of Lemma \ref{the:main}.
\subsection{Proof of Theorem \ref{the:main}}
According to Lemma \ref{the:main} and sentinel strategy, with probability at least $1-\delta$,the upper confidence bound satisfies
\begin{align}
    \bm{x}^{\top}\bm{\theta^*}
    &\leq \bm{x}^{\top}\bm{\hat{\theta}}_{\mathcal{C}^{\bm{x}},t}+\beta_t(\delta)\left\|\bm{x} \right\|_{\bm{V}_{\mathcal{C}^{\bm{x}},t}^{-1}}\\
&=\bm{x}^{\top}\bm{\hat{\theta}}_{\mathcal{C}^{\bm{x}},t}+\beta_t(\delta)\left\|\bm{x} \right\|_{\bar{\bm{V}}_{\mathcal{C}^{\bm{x}},t}^{-1}}\\
&\leq U(\mathcal{C}^{\bm{x}}).
\end{align}
According to the arm selection policy of CS-LB, we have
\begin{align}
\bm{x}_t^{\top}\bm{\hat{\theta}}_{\mathcal{W}_t,t}+\beta_t(\delta)\left\|\bm{x}_t \right\|_{\bar{\bm{V}}_{\mathcal{W}_t,t}^{-1}}=
U(\mathcal{W}_t)=\max\{U(\mathcal{C})\}=U(\mathcal{C}^{\bm{x}_t})=
\bm{x}_t^{\top}\bm{\hat{\theta}}_{\mathcal{C}^{\bm{x}_t},t}+\beta_t(\delta)\left\|\bm{x}_t \right\|_{\bar{\bm{V}}_{\mathcal{C}^{\bm{x}_t},t}^{-1}}.
\end{align}
Then
\begin{align}
    (\bm{x}^*)^{\top}\bm{\theta}^*
    &\leq (\bm{x}^*)^{\top}\bm{\hat{\theta}}_{\mathcal{C}^{\bm{x}^*},t}+\beta_t(\delta)\left\|\bm{x}^* \right\|_{\bar{\bm{V}}_{\mathcal{C}^{\bm{x}^*},t}^{-1}}\\
    &\leq U(\mathcal{C}^{\bm{x}^*})\\
    &\leq U(\mathcal{W}_t)\\
    &= \bm{x}_t^{\top}\bm{\hat{\theta}}_{\mathcal{C}^{\bm{x}_t},t}+\beta_t(\delta)\left\|\bm{x}_t \right\|_{\bar{\bm{V}}_{\mathcal{C}^{\bm{x}_t},t}^{-1}}.
\end{align}
Then, for the per-round regret $regret_t$ in round $t$, we can get
\begin{align}
    regret_t
    &=\max_{\bm{x}\in \mathcal{A}}\left \langle\bm{x},\bm{\theta}^*\right \rangle-\left \langle\bm{x}_t,\bm{\theta}^*\right \rangle\\
    &=\left \langle\bm{x}^*,\bm{\theta}^*\right \rangle-\left \langle\bm{x}_t,\bm{\theta}^*\right \rangle\\
    &\leq \left \langle \bm{x}_t,\bm{\hat{\theta}}_{\mathcal{C}^{\bm{x}_t},t}\right\rangle-\left \langle\bm{x}_t,\bm{\theta}^*\right \rangle+\beta_t(\delta)\left\|\bm{x}_t \right\|_{\bar{\bm{V}}_{\mathcal{C}^{\bm{x}_t},t}^{-1}}\\
    &=\left \langle\bm{x}_t,\bm{\hat{\theta}}_{\mathcal{C}^{\bm{x}_t},t}-\bm{\theta}^*\right \rangle+\beta_t(\delta)\left\|\bm{x}_t \right\|_{\bar{\bm{V}}_{\mathcal{C}^{\bm{x}_t},t}^{-1}}.
\end{align}
According to Lemma \ref{tech0}, we have
\begin{align}
    \left \langle\bm{x}_t,\bm{\hat{\theta}}_{\mathcal{C}^{\bm{x}_t},t}-\bm{\theta}^*\right \rangle\leq \left\|\bm{x}_t \right\|_{\bar{\bm{V}}_{\mathcal{C}^{\bm{x}_t},t}^{-1}}\left\|\bm{\hat{\theta}}_{\mathcal{C}^{\bm{x}_t},t}-\bm{\theta}^* \right\|_{\bar{\bm{V}}_{\mathcal{C}^{\bm{x}_t},t}}\leq \beta_t(\delta)\left\|\bm{x}_t \right\|_{\bar{\bm{V}}_{\mathcal{C}^{\bm{x}_t},t}^{-1}}.
\end{align}
Substituting into the above gives
\begin{align}
    regret_t\leq 2\beta_t(\delta)\left\|\bm{x}_t \right\|_{\bar{\bm{V}}_{\mathcal{C}^{\bm{x}_t},t}^{-1}}.
\end{align}
Since the regret $regret_t\leq 2LS$ and $\beta_t(\delta)>\sqrt{\lambda}S$, then
\begin{align}
    Regret(T)
    &=\sum_{t=1}^{T}regret_t\\
    &\leq 2\sum_{t=1}^{T} \min\{\beta_t(\delta)\left\|\bm{x}_t \right\|_{\bar{\bm{V}}_{\mathcal{C}^{\bm{x}_t},t}^{-1}},LS\} \\
    &\leq 2\max\{1,\frac{L}{\sqrt{\lambda}}\}\beta_T(\delta)\sum_{t=1}^{T}\min\{1,\left\|\bm{x}_t \right\|_{\bar{\bm{V}}_{\mathcal{C}^{\bm{x}_t},t}^{-1}}\}.
\end{align}
On the other hands, we have
\begin{align}
    \sum_{t=1}^{T}\min\{1,\left\|\bm{x}_t \right\|_{\bar{\bm{V}}_{\mathcal{C}^{\bm{x}_t},t}^{-1}}\}=\sum_{\mathcal{C}\in\mathcal{C}^{all}}\sum_{t=1,\bm{x}_t\in \mathcal{C}}^{T}\min\{1,\left\|\bm{x}_t \right\|_{\bar{\bm{V}}_{\mathcal{C},t}^{-1}}\}=\sum_{\mathcal{C}\in\mathcal{C}^{all}}\sum_{t=1,\bm{x}_t\in \mathcal{C}}^{T}\min\{1,\left\|\bm{x}_t \right\|_{\bm{V}_{\mathcal{C},t}^{-1}}\}.
\end{align}
By Cauchy-Schwartz inequality
\begin{align}
    \sum_{t=1,\bm{x}_t\in \mathcal{C}}^{T}\min\{1,\left\|\bm{x}_t \right\|_{\bm{V}_{\mathcal{C},t}^{-1}}\}\leq \sqrt{T_{\mathcal{C},T}\sum_{t=1,\bm{x}_t\in \mathcal{C}}^{T}\min\{1,\left\|\bm{x}_t \right\|_{\bm{V}_{\mathcal{C},t}^{-1}}^2\}}.
\end{align}
Since $rank(\bm{X}_{\mathcal{C},t}^{\top}\bm{X}_{\mathcal{C},t})=rank(\bm{S}_{\mathcal{C},t}^{\top}\bm{S}_{\mathcal{C},t})\leq l$ for each cluster $\mathcal{C}$, we have $\varepsilon_l=0$ within that cluster. Thus according to Lemma \ref{tech2} we have
\begin{align}
    \sum_{t=1,\bm{x}_t\in \mathcal{C}}^{T}\min\{1,\left\|\bm{x}_t \right\|_{\bm{V}_{\mathcal{C},t}^{-1}}^2\}\leq 2l\ln\left (1+\frac{tL^2}{l\lambda} \right ).
\end{align}
Combining the above inequalities yields
\begin{align}
    \sum_{t=1}^{T}\min\{1,\left\|\bm{x}_t \right\|_{\bar{\bm{V}}_{\mathcal{C}^{\bm{x}_t},t}^{-1}}\}
    \leq \sqrt{2l\ln\left (1+\frac{tL^2}{l\lambda} \right )}\sum_{\mathcal{C}\in\mathcal{C}^{all}}\sqrt{T_{\mathcal{C},T}}.
\end{align}
By Cauchy-Schwartz inequality and the fact that $\sum_{\mathcal{C}\in\mathcal{C}^{all}}T_{\mathcal{C},T}=T$, we can get
\[
\sum_{\mathcal{C}\in\mathcal{C}^{all}}\sqrt{T_{\mathcal{C},T}}\leq \sqrt{N_{\mathcal{A}}^lT}.
\]
At the same time, we known that $\beta_T(\delta)=\widetilde{O}(\sqrt{l})$ and we set $\lambda\geq \max\{1,L^2\}$; with probability at least $1-\delta$, we get that the regret is
\begin{align}
    Regret(T)\leq 2\beta_T(\delta)\sqrt{2lN_{\mathcal{A}}^lT\ln\left (1+\frac{TL^2}{l\lambda} \right )}=\widetilde{O}(l\sqrt{N_{\mathcal{A}}^lT}).
\end{align}
Thus, we complete the proof.
\end{document}